\documentclass[11pt]{article}
\usepackage{fullpage}

\usepackage[utf8]{inputenc} 
\usepackage[T1]{fontenc}    
\usepackage{hyperref}       
\usepackage{url}            
\usepackage{booktabs}       
\usepackage{amsfonts}       
\usepackage{nicefrac}       
\usepackage{microtype}      
\usepackage{enumerate}
\usepackage[shortlabels]{enumitem}
\usepackage{algorithmic}
\usepackage{algorithm}
\usepackage{subfigure}
\usepackage{graphicx} 
\usepackage{caption}
\usepackage{amsmath}
\usepackage{amsthm}
\usepackage{amssymb}
\usepackage{tikz}
\usepackage{multirow}
\usepackage{enumerate}
\usepackage{color}
\usepackage{xcolor}
\usepackage[round]{natbib}
\usepackage{verbatim}
\usepackage{xcolor}
\usepackage{amsmath}
\usepackage{amssymb}
\usepackage{bm}
\usepackage{hyperref}
\usepackage{algorithm}
\usepackage{algorithmic}
\usepackage{subfigure}
\usepackage{graphicx}
\usepackage{multirow}
\usepackage[all]{hypcap}
\usepackage{mathrsfs}
\usepackage{hyperref}
\usepackage{bm,todonotes}
\usepackage{footnote}

\usetikzlibrary{arrows}

\allowdisplaybreaks[4]
\allowdisplaybreaks

\definecolor{darkgreen1}{rgb}{0,0.7,0.3}

 \hypersetup{
     colorlinks=true,
     linkcolor=magenta,
     filecolor=blue,
     citecolor = green!60!black,
     urlcolor=cyan,
     }

\newcommand*{\QEDA}{\hfill\ensuremath{\blacksquare}}

\newtheorem{thm}{Theorem}

\newtheorem{lem}[thm]{Lemma}
\newtheorem{prop}[thm]{Proposition}
\newtheorem{defn}{Definition}
\newtheorem{rem}{Remark}
\newtheorem{exam}{Example}

\newtheorem*{sche}{Sampling and Aggregation}
\newtheorem{fact}{Fact}

\begin{document}

\title{\LARGE \textbf{FedPower: Privacy-Preserving Distributed Eigenspace Estimation~\thanks{The shorter version of this paper is published in \emph{Proceedings of the 38th International Conference on Machine Learning}, PMLR 139:6504-6514.}
\thanks{Xiao Guo and Xiang Li make an equal contribution to this paper. Xiangyu Chang is the corresponding author (\texttt{xiangyuchang@xjtu.edu.cn}).}}}
\author{ Xiao Guo$^1$, Xiang Li$^2$, Xiangyu Chang$^3$, Shusen Wang$^4$, Zhihua Zhang$^2$\hspace{.2cm}  \\
    $1$ School of Mathematics, Northwest University, China\\
    $2$  School of Mathematical Sciences, Peking University, China \\
    $3$ School of Management, Xi'an Jiaotong University, China\\
    $4$ Xiaohongshu Inc., China
    }

\maketitle

\begin{abstract}%
    Eigenspace estimation is fundamental in machine learning and statistics, which has found applications in PCA, dimension reduction, and clustering, among others.
    The modern machine learning community usually assumes that data come from and belong to different organizations. The low communication power and the possible privacy breaches of data make the computation of eigenspace challenging. To address these challenges, we propose a class of algorithms called \textsf{FedPower} within the federated learning (FL) framework. \textsf{FedPower} leverages the well-known power method by alternating multiple local power iterations and a global aggregation step, thus improving communication efficiency. In the aggregation, we propose to weight each local eigenvector matrix with {\it Orthogonal Procrustes Transformation} (OPT) for better alignment. To ensure strong privacy protection, we add Gaussian noise in each iteration by adopting the notion of \emph{differential privacy} (DP). We provide convergence bounds for \textsf{FedPower} that are composed of different interpretable terms corresponding to the effects of Gaussian noise, parallelization, and random sampling of local machines. Additionally, we conduct experiments to demonstrate the effectiveness of our proposed algorithms.


\end{abstract}


\section{Introduction}\label{sec:introduction}
Modern machine learning tasks involve massive data that come from different sources, such as hospitals, banks, companies, etc.
The communication power is thus limited, which makes large-scale data applications challenging. Further, the data from local sources often contain sensitive information about individuals, hence, privacy issues become more and more prominent \citep{bhowmick2018protection,dwork2014algorithmic}.

\emph{Federated learning} (FL) has emerged as a prominent paradigm for distributed learning in large-scale problems involving multi-source data. For further background and recent advancements, please refer to \citet{kairouz2021advances} and related literature. In FL, each machine typically trains its model locally and sends parameter updates to the central server whenever communication is required. The server then aggregates these updates, potentially with randomization, and broadcasts them to synchronize all local parameters. This process is iterated until convergence or the fulfillment of specific conditions. To address the aforementioned challenges including large-scale data and unreliable communication, autonomy, and privacy issues; {see \citet{mcmahan2017communication,smith2017federated,sattler2019robust,li2020challenges}, a good FL algorithm should meet the following requirements. First, a FL algorithm should be communication efficient with more local computations and fewer communications. Second, it should be able to deal with the scheme when some of the local machines are inactive or the organizations decide not to participate in the following training procedures. Third, individual's privacy should be protected, which is ensured apparently by leaving the original data at local machines.

Nevertheless, even if only the updates but not the original data are transmitted to the central server, individuals' privacy can still be compromised through via delicately designed attacks \citep{dwork2017exposed,melis2018inference,zhou2020differentially}. To address this issue, the framework of \emph{differential privacy} (DP) \citep{dwork2006calibrating,dwork2014algorithmic} has gained widespread attention in private data analysis. A differentially private (DP) algorithm pursues that if the data is changed by one row (entry) with pre-specified limits, then the algorithm's output appears similar in probability. Such algorithms protect the individuals' privacy from any adversary who knows the algorithm's output and even the rest of the data and can resist any kind of attack. Typically, a differentially private algorithm is obtained by adding calibrated noise to the non-differentially private algorithm.

This paper focuses on the problem of eigenspace estimation within the private-preserving federated learning framework. Eigen-decomposition is a widely used technique in various machine learning tasks, such as dimension reduction \citep{wold1987principal}, clustering \citep{von2007tutorial}, ranking \citep{negahban2017rank}, matrix completion \citep{candes2009exact}, multiple testing \citep{fan2019farmtest}, and factor analysis \citep{bai2013principal}. It also finds applications in fields like finance, biology, and neurosciences \citep{izenman2008modern}. The computation of eigenvectors has evolved since the 1960s, with seminal works by \citet{golub1965calculating,golub1970singular} that provided the basis for the \emph{EISPACK} and \emph{LAPACK} routines. For computing leading eigenspace of matrices, iterative algorithms such as the power iteration and its variants \citep{Golub2012matrix,hardt2014noisy} flourished. Recently, to solve large-scale problems, distributed learning of eigenvectors or principle components is receiving more and more attention; see \citet{fan2019distributed,chen2021distributed}, among others. However, as far as we are aware, most existing works can not meet the aforementioned challenges simultaneously.

To tackle the challenges of large-scale computation, unreliable communication, and privacy breaches in modern data analysis, we propose a set of algorithms called the \emph{Federated Power} method (\textsf{FedPower}). Building upon the well-known single-machine power method, \textsf{FedPower} assumes a distributed data setting, where each machine performs local power iterations using its data. After several local steps, the local machines send their updates to the central server, which aggregates and returns the results back to the machines. Due to the orthogonal ambiguity of subspaces, we employ the \emph{Orthogonal Procrustes Transformation} (OPT) during the aggregation. As connectivity issues may occur during the training process where each local machine may lose connection to the server actively or passively, we present two protocols: \emph{full participation} and \emph{partial participation}. In the partial participation protocol, the server may collect the first few responded local machines within a certain time range. We model this by assuming that the local machines are sampled with the replacement for fixed times.
Moreover, to avoid privacy leakage, we take advantage of the notion of DP to add Gaussian noise to the updates in each iteration. This is based on our assumption that the server is \emph{honest-but-curious (semi-honest)}.


In addition to the algorithms, we study how the \textsf{FedPower} performs theoretically. Firstly, we provide rigorous analysis demonstrating the DP guarantees of the algorithms corresponding to both the full and partial participation schemes. These guarantees are established using the notion of \emph{R\'{e}nyi differential privacy} \citep{mironov2017renyi}.
Secondly, we analyze the
convergence bound of \textsf{FedPower} in terms of the subspace distance between the estimated and the true eigenspace using advanced random matrix theory and delicate error decomposition techniques.
For the full participation scheme, the convergence error comprises two components. One component arises from the Gaussian noise, while the other comes from the parallelization and synchronization. For the partial participation scheme, the resulting convergence error bound consists of three components. Besides the two components appearing in the error of the full participation scheme, there exists an additional component that comes from the sampling of local machines, which could be regarded as the sampling bias term. The more local machines that are sampled, the smaller the bias term would be. As expected, the final error bounds can be made sufficiently small by ensuring a sufficiently large-sample and high-quality local data.

The remainder of the paper is organized as follows. Section \ref{sec:preliminary} introduces the power method for the centralized and distributed eigenspace estimation, and two notions of DP. Section \ref{sec:method} includes the proposed algorithms \textsf{FedPower} and the corresponding convergence analysis under two schemes, namely, the full participation and the partial participation. Section \ref{sec:related} reviews and discusses the related works, and also summarizes the main contributions of this work. Section \ref{sec:experiment} presents the experimental results. Section \ref{sec:conclusion} concludes the paper. Technical proofs and supplementary materials are all included in the Appendix.


\section{Preliminaries}
\label{sec:preliminary}
In this section, we first present the power method for computing the eigenspace, and a naive distributed power method designed for the distributed eigenspace computation. Then, we pose the demand, namely, the communication efficiency and privacy preservation, that the modern machine learning tasks call for. In particular, we propose one adversary model to show how privacy can be leaked during the communication rounds of the naive distributed power method. Last, we present the preliminaries of $(\epsilon,\delta)$-DP and its variant R\'{e}nyi differential privacy.


\subsection{Eigen-decomposition and Power Method}
Given a positive semi-definite matrix $A\in \mathbb R^{d\times d}$, its full eigen-decomposition is defined as
$${A=U{\Sigma} U^\intercal =\sum_{i=1}^d {\sigma_i}u_iu_i^\intercal,}$$
where $U=[u_1,u_2,\dots,u_d]\in\mathbb R^{ d\times d}$ are orthogonal matrices that contain the eigenvectors of $A$, and $\Sigma={\rm diag}\{\sigma_1,\sigma_2, \ldots,\sigma_d\}$ is a diagonal matrix with the eigenvalues in decreasing order on the diagonal. The partial or truncated eigen-decomposition aims to compute the top $k~(k\leq d)$ eigenvectors $U_k=[u_1, \ldots, u_k]$ and use the truncated decomposition $U_k\Sigma_k U_k^\intercal$ to approximate $A$, where $\Sigma_k={\rm diag}\{\sigma_1, \ldots, \sigma_k\}$.

The power method \citep{Golub2012matrix} computes $U_k$ by iterating
\begin{equation}
\phantomsection
\label{2.1}
Y\leftarrow A Z \quad{\rm and}\quad Z\leftarrow {\rm \textsf{orth}}(Y),
\tag{2.1}
\end{equation}
where $Y$ and $Z$ are $d\times k$ matrices, and ${\rm \textsf{orth}}(Y)$ means orthogonalizing the columns of $Y$ via QR-factorization.

\subsection{Distributed Power Method}
\label{subsec: distributed}
Suppose $A$ is stored at $m$ different local machines such that
\begin{equation}
\label{2.2}
A=\frac{1}{m}\sum_{i=1}^m A_i.
\tag{2.2}
\end{equation}
Typically, we consider the scenario where $A=\frac{1}{n}M^\intercal M$ for some $M\in \mathbb R^{n\times d}$. Suppose that $M$ is partitioned to $m$ blocks by row such that $M^\intercal =[M_1^\intercal,...,M_m^\intercal]$, where $M_i\in \mathbb R ^{s_i\times d}$ includes $s_i$ rows of $M$ and $\sum_{i=1}^m s_i=n$. For simplicity, we assume that $M$ is equally partitioned and $s_i=\frac{n}{m}$ for all $i$'s, yet we note that the results can be easily extended to more general cases. It is then clear that
\begin{equation}
\label{2.3}
A_i=\frac{m}{n}M_i^\intercal M_i.
\tag{2.3}
\end{equation}
For notational simplicity, we refer $A_i$ in what follows but one should keep (\ref{2.3}) in mind.

With (\ref{2.2}), $Y$ in (\ref{2.1}) can be written as
\begin{equation*}
Y=\frac{1}{m}\sum_{i=1}^m A_i Z\in\mathbb R^{d\times k},
\end{equation*}
which implies that the power method can be parallelized. See Figure \ref{dispowerpic} and Algorithm \ref{dispower}, which is called the distributed power method. Note that Algorithm \ref{dispower} is identical to the power method except that the summations therein come from different workers. The following theorem is a well-known result on the convergence of the power method as well as the distributed power method \citep{arbenz2012lecture}.
\begin{thm}
\label{dispowerthm}
{Let $\sigma_k$ be the $k$-th largest eigenvalue of $A$ and assume $\sigma_{k+1}>0$, where $1\leq k< d$.} Then for any $\epsilon>0$, with high probability, after $T=O(\frac{\sigma_k}{\sigma_{k+1}}{\rm log}(\frac{d}{\epsilon}))$ iterations, the output ${Z}_T$ of Algorithm \ref{dispower} satisfies
$${\rm sin}\theta_k({Z}_T,U_k)=\|(\mathbb I_d-{Z}_T{Z}_T^\intercal)U_k\|_2\leq \epsilon,$$
where ${\rm sin}\theta_k$ denotes the $k$-th \emph{principle angles} between two subspaces which can be regarded as a subspace distance \footnote{The formal definition can be found in Section \ref{sec:notes}.}, and $\mathbb I_d$ denotes the identity matrix of dimension $d$.
\end{thm}

The distributed power method can deal with data distributed across multiple workers. However, it falls short of addressing the challenges and concerns encountered in modern data applications. Firstly, Algorithm \ref{dispower} necessitates two communication rounds in each iteration, resulting in substantial communication costs when simple parallelizing the power method. Furthermore, the issues of the straggler's effect and privacy breaches are yet to be resolved. In the subsequent subsections, we will delve into further explanations of these two concerns.

\begin{figure*}[!htbp]{}
\centering
{\includegraphics[height=5.6cm,width=11.2cm,angle=0]{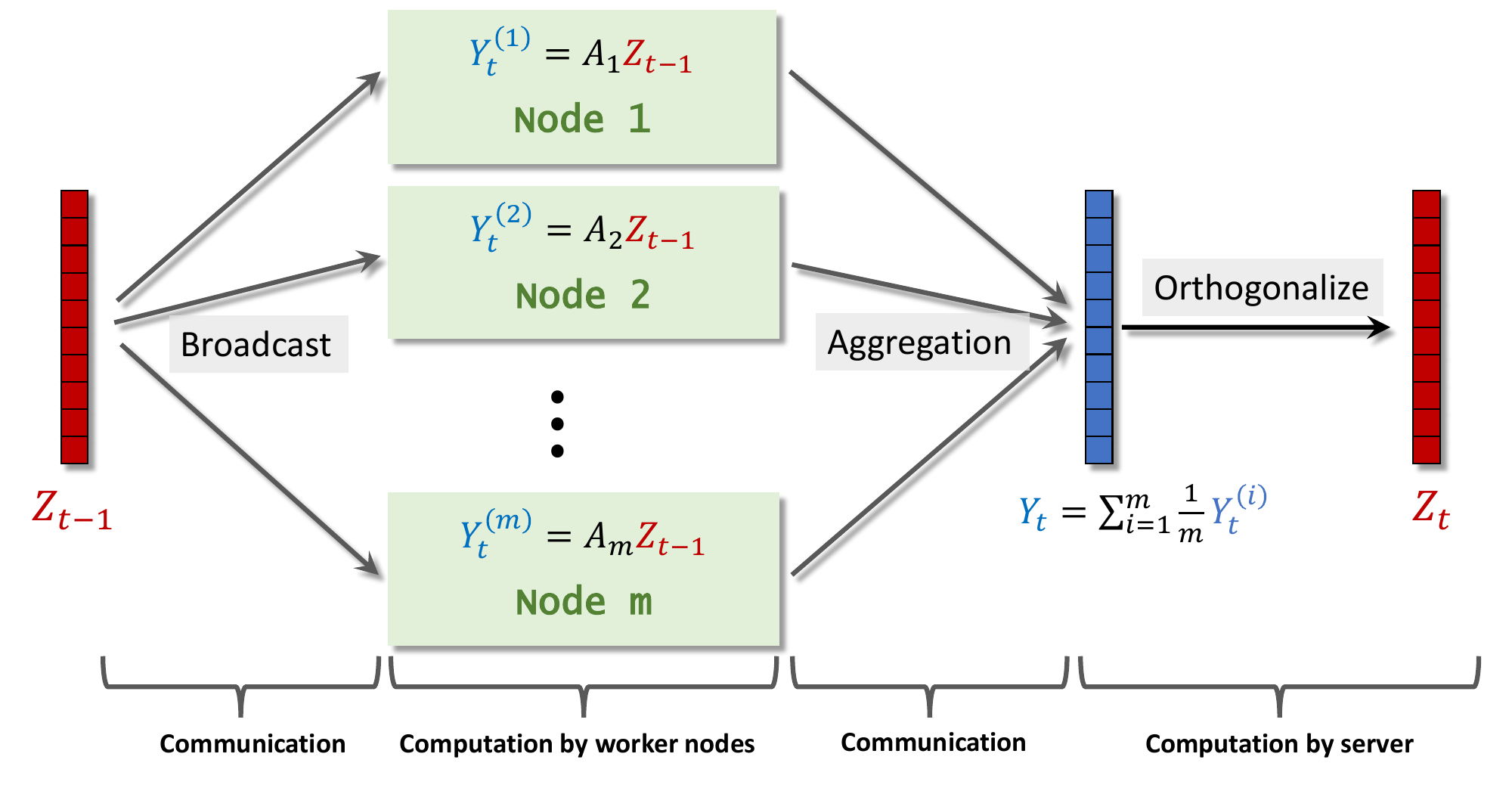}}
\caption{Protocol of the distributed power method. In every iteration, there are two rounds of communications.}\label{dispowerpic}
\end{figure*}

\begin{algorithm}[htb]
\small

\renewcommand{\algorithmicrequire}{\textbf{Input:}}

\renewcommand\algorithmicensure {\textbf{Output:} }

\caption{Distributed Power Method }

\label{dispower}

\begin{algorithmic}[1]
\STATE \textbf{Input:} distributed dataset $\{A_i\}_{i=1}^m$, target rank $k$, number of iterations $T$. \\
\STATE \textbf{Initialization:} orthonormal $Z_0^{(i)}=Z_0\in \mathbb R^{d\times k}$ by QR decomposition on a random Gaussian matrix with i.i.d. entries from $\mathcal N(0,1)$.\\
\FOR{$t = 1$ to $T$}
\STATE The $i$-th worker independently performs $Y_t^{(i)}=A_iZ_{t-1}^{(i)}$ for all $i\in [m]$;\\
\STATE Each worker $i$ sends $Y_t^{(i)}$ to the server and the server performs aggregation: $Y_t=\sum_{i=1}^m \frac{1}{m} Y_t^{(i)}$;\\
\STATE The server performs orthogonalization: $Z_t={\rm \textsf{orth}}(Y_t)$ and broadcast $Z_t$ to each worker such that $Z_{t}^{(i)}=Z_t$;
\ENDFOR
\STATE \textbf{Output:} approximated eigen-space ${Z}_{T}\in\mathbb R^{d\times k}$ with orthonormal columns.
\end{algorithmic}
\end{algorithm}

\subsection{Straggler's Effect}

Unlike traditional distributed learning, the FL system operates under the premise that the server does not have control over the local machines, leading to potential disconnections between the server and the local machines \citep{kairouz2021advances,li2020convergence}. This introduces two key considerations. Firstly, certain local machines may experience issues such as being powered off, encountering technical failures, or having limited internet connectivity, thereby becoming stragglers. Secondly, each local machine retains its autonomy, with the owners having the discretion to opt out of specific training steps for various reasons. Consequently, waiting for responses from all local machines becomes unfeasible for the server. Instead, the server can utilize the updates from the first few responsive local machines within a predetermined timeframe. Our proposed algorithm effectively captures and accounts for the impact of stragglers.

\subsection{Adversary Model}
\label{adv}

Though the data are not shared by all the participants in  FL or more general distributed learning systems, privacy breaches remain possible. In this paper, we consider a specific adversary termed as the \emph{curious onlooker}, who can eavesdrop on the communication between the server and the local machines and possesses knowledge of the learning tasks and protocols. We consider the server as a potential onlooker, characterized as honest-but-curious (semi-honest). This means that while the server does not violate the protocol to attack the raw data, it is curious and will attempt to learn all possible information from its received messages \citep{goldreichfoundations}. For instance, in a company, internal employees responsible for model training may attempt to infer personal information about the users. Additionally, we assume that each local machine is honest and does not attempt to infer information from other machines. The following example demonstrates how the distributed power method (Algorithm \ref{dispower}) can potentially cause privacy threats to curious onlookers.

\begin{exam}[Privacy breaches via curious onlookers]
\label{attack1}
Consider Algorithm \ref{dispower} and assume that the server knows the updates $Y_t^{(i)}$ for each $i\in [m]$ and $t\in T$.
In addition, the server could infer $Z_t^{(i)}$ from $Y_t^{(i)}$ because it also knows the learning rule that $Z_t^{(i)}$ is obtained from $Y_t^{(i)}$ via the QR decomposition (Line 6 in Algorithm \ref{dispower}). Then by
\begin{equation}
\phantomsection
\label{2.4}
Y_t^{(i)}=A_iZ_{t-1}^{(i)}\; {\mbox{(Line 4 in Algorithm \ref{dispower})}},
\tag{2.4}
\end{equation}
and the results in Theorem \ref{dispowerthm} that the distributed power method converges after $T=O(\frac{\sigma_k}{\sigma_{k+1}}{\rm log}(\frac{d}{\epsilon}))$ iterations, the server can infer $A_i$ using $d\times k\times T$ equations provided that there are at most $O(d \log d)$ unknown elements in $A_i$. Indeed, the server may participate in the data collection, therefore they may know some prior information about $A_i$.

\end{exam}

We have presented a simple example to illustrate the potential privacy leakage that can occur through curious onlookers. It is important to note that more sophisticated attacks can be found in studies such as~\citet{madry2017towards,chakraborty2018adversarial}, among others.


\subsection{Differential Privacy}


We have seen that privacy concerns are prevailing in modern data analysis. How to quantitatively describe privacy is the key point to understanding and designing privacy-preserving algorithms. Differential privacy, first introduced in \citet{dwork2006calibrating}, is a rigorous and most widely adopted notion of privacy, which generally guarantees that a randomized algorithm behaves similarly on similar input databases.
The $(\varepsilon,\delta)$-DP \citep{dwork2014algorithmic} is defined as follows. Throughout this paper, we use the term "DP" as an abbreviation for ``differential privacy'' or ``differentially private'', following common usage in the field.

\begin{defn}[$(\varepsilon,\delta)$-DP]
\label{dp}
A randomized algorithm $\mathcal M$: $\mathcal X^n\rightarrow \Theta$ is called $(\varepsilon,\delta)$-DP if for all pairs of \emph{neighboring databases} $X, X'\in \mathcal X^{n}$, and for all subsets of range $S\subseteq \Theta$:
$$\mathbb P(\mathcal M(X)\in S)\leq {\rm exp}(\varepsilon)\mathbb P(\mathcal M(X')\in S)+\delta.$$
\end{defn}



In literature, the definition of neighboring datasets $X, X'$ varies case by case.
In this work, dataset $X$ is the data matrix $M_i$ introduced in Section \ref{subsec: distributed}, and we consider the scenario where ``neighboring'' means $M_i$ and $M'_i$ differ in one row, with a Euclidean norm of 1.
DP achieves the privacy goal that anything can be learned about an individual from the released information can also be learned without that individual's participation. The $\varepsilon$ is often called the privacy budget which is a small constant measuring the privacy loss and it should be no larger than 1 typically \citep{dwork2014algorithmic}. The $\delta$ is also a small constant and it can be thought of as a tolerance of the more stringent $\varepsilon$-DP, i.e., $(\varepsilon,\delta)$-DP with $\delta$ being 0.

While $(\varepsilon,\delta)$-DP is interpretable and user-friendly, it can not tightly handle composition, that is, how much the privacy loss accumulates under repeated quires on the same data. To ensure the exact composition, many new notions of DP have been developed, such as concentrated differential privacy \citep{dwork2016concentrated}, R\'{e}nyi differential privacy \citep{mironov2017renyi}, truncated concentrated differential privacy \citep{bun2018composable}, and Gaussian differential privacy \citep{dong2019gaussian}, among others. In this work, we specifically focus on R\'{e}nyi differential privacy (RDP).

\begin{defn}[$(\alpha, \varepsilon)$-RDP]
A randomized algorithm $\mathcal M$: $\mathcal X^n\rightarrow \Theta$ is called $\epsilon$-R\'{e}nyi DP of order $\alpha$, or $(\alpha, \varepsilon)$-RDP, if for all pairs of \emph{neighboring databases} $X, X'\in \mathcal X^{n}$,
\begin{equation}
D_\alpha (\mathcal M (X)\|\mathcal M (X')):=\frac{1}{\alpha-1} \log \mathbb E_{x\sim \mathcal M(X')} \left[\left(\frac{\mathcal M(X)(x)}{\mathcal M(X')(x)}\right)^\alpha\right]\leq \varepsilon.\nonumber
\end{equation}
\end{defn}
RDP can be achieved using the following Gaussian mechanism \citep{mironov2017renyi}.
\begin{prop}[RDP Gaussian mechanism]
\label{gaussrdp}
For any $f: \mathcal X^n\rightarrow \mathbb R^d$, its $l_2$-sensitivity is defined by
\[\triangle_2(f)=\underset{X,X'\,\mbox{\tiny neighboring}}\sup\| f(X)- f(X')\|_2.
\]
If ${\triangle_2(f)<\Delta}$,
then the Gaussian mechanism given by \[\mathcal M(X):=f(X)+(\xi_1,\xi_2,\dots,\xi_d)^\intercal,
\]
where the $\xi_i$ are  i.i.d.~ drawn from {$\mathcal N(0,\nu^2)$}, achieves $(\alpha,\frac{\alpha \Delta^2}{2\nu^2})$-RDP for any $\alpha>1$.
\end{prop}

RDP enjoys the following two nice properties \citep{mironov2017renyi}. One is the composition property, that is, the privacy degrade of repeated mechanisms is just the summation of the privacy budgets of each mechanism. The other is the post-processing property, meaning that an $(\alpha,\varepsilon)$-RDP algorithm is still $(\alpha,\varepsilon)$-RDP after any post-processing procedure provided that no additional knowledge about the database is used.


\begin{prop}[Adaptive composition of RDP]
\label{composition}
For $\alpha \in (1,\infty)$, suppose $\mathcal M_1: \mathcal X^n\rightarrow \Theta_1$ is $(\alpha, \varepsilon_1)$-RDP and $\mathcal M_2: \mathcal X^n\times \Theta_1\rightarrow \Theta_2$ is $(\alpha, \varepsilon_2)$-RDP, which takes as input the output of the first mechanism $\mathcal M_1$ in addition to the dataset. Then, the joint mechanism $\mathcal M: \mathcal X^n \rightarrow \Theta_1\times \Theta_2$ defined as
$$\mathcal M(X)=(\theta_1, \mathcal M_2 (X,\theta_1))$$
achieves $(\alpha, \varepsilon_1+\varepsilon_2)$-RDP, where $X\subset \mathcal X^n$ and $\theta_1=\mathcal M_1(X)$.
\end{prop}
The composition of RDP can be generalized to multiple compositions.
\begin{prop}[Post-processing of RDP]
\label{post}
{Let $\mathcal M: \mathcal X^n\rightarrow \Theta$ be an $(\alpha,\varepsilon)$-RDP algorithm, and $g:\Theta\rightarrow \Theta'$ be an arbitrary randomized mapping. Then $g\circ\mathcal M :\mathcal X^n\rightarrow \Theta'$ is $(\alpha,\varepsilon)$-RDP.}
\end{prop}


We also have the conversion of RDP to $(\varepsilon,\delta)$-DP to enhance the interpretability of DP.

\begin{prop}[RDP to DP]
\label{rdptodp}
If $\mathcal M$ is $(\alpha, \varepsilon)$-RDP, then $\mathcal M$ is $(\varepsilon + \frac{\log (1/\delta)}{\alpha-1}, \delta)$-DP for $\forall 0<\delta<1$.
\end{prop}

All the properties enhance the utility of RDP in practical applications.

\subsection{Notation}
\label{sec:notes}

We summarize the notation and notions  used in the following sections of this paper. Given a target matrix $A \in \mathbb{R}^{d{\times} d}$,
the $k$-th largest eigenvalue of $A$ is denoted by $\sigma_k$. Let $k$ and $r\ (r\geq k)$ be the target rank and iteration rank of partial eigen-decomposition, respectively.  The numbers of total iterations is denoted by $T$ and the set $\{1, \ldots, T\}$ is denoted by $[T]$. $\|\cdot\|_2$ denotes the spectral norm of a matrix or the Euclidean norm of a vector, $\|\cdot\|_{\tiny {\rm max}}$ denotes the entry-wise maximum absolute value of a matrix or a vector, $\|\cdot\|_\infty$ denotes the matrix operator $\ell_\infty$ norm, and $\|\cdot\|_{\tiny {\rm m}}$ denotes the minimum singular value of a matrix. Let $\kappa =\|A\|_2\|A^\dagger\|_2$ denote target matrix $A$'s condition number. $\mathcal O_r$ denotes the set of $r\times r$ orthogonal matrices and $\mathbb I_r$ denotes the identity matrix with dimension $r$.

{In addition, we use the following standard notation for asymptotics. We write $f(n)\asymp g(n)$ if $c g(n)\leq f(n)\leq C g(n)$ for some constants $0<c<C<\infty$. $f(n)\lesssim g(n)$ or $f(n)=O(g(n))$ if $f(n)\leq Cg(n)$ for some constant $C<\infty$. $f(n)=\Omega(g(n))$ if $f(n)\geq cg(n)$ for some constant $c>0$.} Finally, we provide the definition of \emph{projection distance}, which measures the distance of two subspaces.

\begin{defn}[Projection distance]
\label{dfdis}
Given two column-orthonormal matrices $U,\tilde{U}\in \mathbb R^{d\times k}$, the projection distance between the two subspaces spanned by their columns is defined as
\begin{equation}
\phantomsection
\label{2.6}
{\rm dist}(U,\tilde{U}):=\|UU^\intercal-\tilde{U}\tilde{U}^\intercal\|_2=\|\tilde{U}^\intercal U^\bot\|_2=\|{U}^\intercal \tilde{U}^\bot\|_2={\rm sin }\,\theta_k(U,\tilde{U}),
\tag{2.6}
\end{equation}
where $U^\bot$ and $\tilde{U}^\bot$ denote the complement subspaces of $U$ and $\tilde{U}$, respectively.
Here $\theta_k$ denotes the $k$-th {principle angle} between two subspaces; see Appendix E for the formal definition.
\end{defn}

\section{Privacy-Preserving Distributed Eigenspace Estimation}
\label{sec:method}
In this section, we develop a set of power-iteration-based algorithms, called \textsf{FedPower}, for the eigenvector computation which is communication efficient, privacy-preserving, and allows for partial participation of the local machines. In particular, depending on whether the local machines all participate in the communication rounds, we study two protocols, namely, the full participation protocol and the partial participation protocol. The privacy bounds and convergence rates will also be established.


Before delving into the details, let us provide an overview of the basic idea behind \textsf{FedPower}, as depicted in Figure \ref{fedpowerpic}. The setup is consistent with that described in Section \ref{subsec: distributed}. To improve the communication efficiency of the naive distributed power method, \textsf{FedPower} trades more local computations for fewer communications. Specifically, each worker locally runs $$Y_t^{(i)}=A_iZ_{t-1}^{(i)}$$ multiple times between two communication rounds. Let $T$ be the total number of iterations performed by each worker. Let $\mathcal I_T$, a subset of $[T]$, index the iterations that call for communications. $|\mathcal I_T |$ denotes its cardinality. If $\mathcal I_T = [T]$, synchronization happens at every iteration as in the distributed power method (see Figure \ref{dispowerpic}). If $\mathcal I_T = \{ T \}$, synchronization happens only at the end, and \textsf{FedPower} is similar to the one-shot divide-and-conquer eigenspace \citep{fan2019distributed}.  An important example that we will focus on latter is $\mathcal I_T^p$. It is defined by
\begin{equation}
\phantomsection
\label{3.1}
\mathcal I_T^p = \{ t \in [T]: t  \ \text{mod}  \ p = 0  \}
= \{ 0, p, 2p, \cdots, p \lfloor T/p \rfloor \},
\tag{3.1}
\end{equation}
where $p$ is a positive integer and $\lfloor T/p \rfloor$ is the largest integer which is smaller than $T/p$. \textsf{FedPower} with $\mathcal I_T^p$ only performs communications every $p$ iterations.

In order to address the orthogonal ambiguity of subspaces, the central server performs an average of the orthogonally transformed local eigenvectors during communication, denoted as $Y_{t}^{(i)}D_{t}^{(i)}$, instead of using the raw $Y_{t}^{(i)}$ values.
The orthogonal matrices $D_{t}^{(i)}$ are constructed using the following procedure.  First, we choose a baseline local machine that to be aligned with. Without loss of generality, we assume the first machine is chosen in the following paper. Next, we compute

\begin{equation}
\label{3.2}
D_{t}^{(i)}=\underset{D\in \mathcal F \cap\mathcal O_r }{{\rm argmin}}\;\|Z_{t-1}^{(i)}D-Z_{t-1}^{(1)}\|_F,
\tag{3.2}
\end{equation}
where $\mathcal O_{r}(r\geq k)$ denotes the set of $r\times r$ orthogonal matrices.
The choice of $\mathcal F$ can vary depending on the specific scenario.
When $\mathcal F=\{\mathbb I_r\}$, the feasible set contains a single element, and as a result, $D_{t}^{(i)}=\mathbb I_r$.
When $\mathcal F=\mathcal O_r$, (\ref{3.2}) represents the classical matrix approximation problem known as the \emph{Procrustes problem} \citep{schonemann1966generalized,cape2020orthogonal}. The solution to (\ref{3.2}), referred to as the \emph{Orthogonal Procrustes Transformation} (OPT), can be obtained in closed form:
\begin{equation*}
    D_{t}^{(i)} \; = \; W_1W_2^\intercal,
\end{equation*}
where we assume that the SVD of $(Z^{(i)}_{t-1} )^\intercal Z^{(1)}_{t-1} $ is $W_1 \Lambda W_2^\intercal$.
The concurrent work \citep{charisopoulos2021communication} also uses OPT.


Recall that we consider a strong adversary termed as honest-but-curious onlooker (see Section \ref{adv}) who does not violate the rules to peep the raw data but is curious to infer data information from the communicative messages and the training rule. To prevent the potential privacy breaches, we consider the \emph{record-level} DP that protects each row of one device's data $M_i$'s which has Euclidean norm 1; recall also (\ref{2.3}). We note that the \emph{client-level} DP that protects against the whole data of each device has also been studied in FL \citep{mcmahan2018learning}. However, this kind of privacy guarantee is too strong if the devices represent different organizations, say hospitals, companies, banks, etc \citep{li2019differentially,zheng2021federated}. In the record-level DP's, we add calibrated Gaussian noise after each local steps on each active devices. After tight composition analysis of DP, the procedure can be shown to be DP at given privacy budget.

Finally, to take into account the straggler's effect, we also consider the partial participation protocol, that is, each aggregation only involves the first $K$ responded (not necessarily different) machines before a certain time. Specifically, we model such a setting using sampling with replacement strategy.

\begin{figure*}[htb]{}
\centering
\subfigure{\includegraphics[height=6.8cm,width=13cm,angle=0]{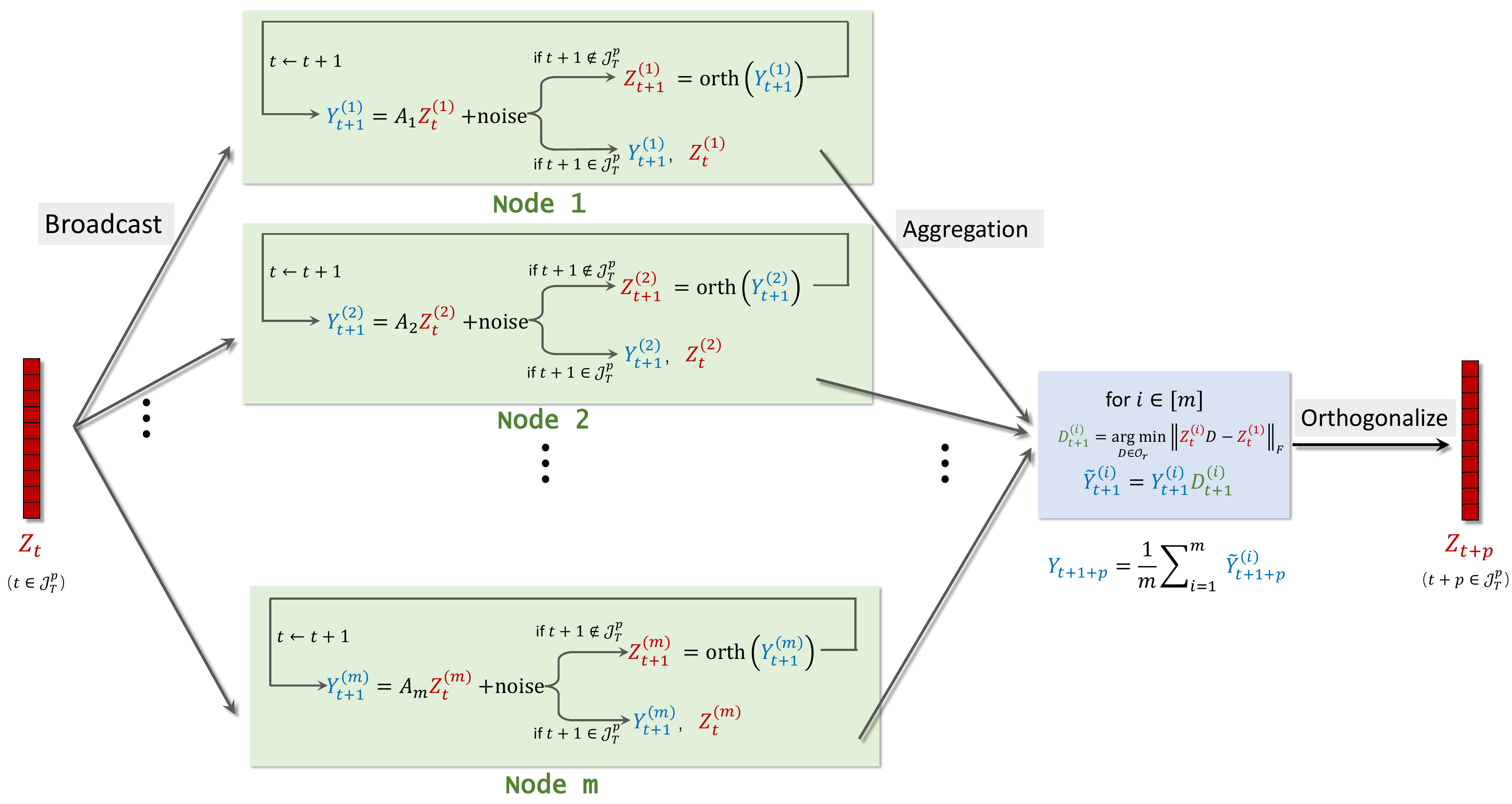}}
\caption{The full participation protocol of \textsf{FedPower} involves each local machine conducting noisy power iterations locally until a communication step is triggered. When communication occurs ($t+1 \in \mathcal I_{T}^p$), the local machines transmit the updates $Y_{t+1}^{(i)}$ and previous $Z_{t}^{(i)}$ to the server. Subsequently, the server performs alignment and aggregation procedures. }\label{fedpowerpic}
\end{figure*}

\subsection{Federated Power Method under Full Participation Protocol}

 The \textsf{FedPower} under the full participation protocol is shown in Algorithm \ref{fedpowerfull}, where we allow the iteration rank $r$ (i.e., the dimension of the output) being equal or larger than the target rank $k$, and we use $\mathcal N(0, \nu^2)^{d\times r}$ to denote a $d\times r$ random matrix with each entry being i.i.d. $\mathcal N(0,\nu^2)$. We have the following remark on Algorithm \ref{fedpowerfull}.

\begin{rem}
If the data of the first local machine is publicly available, the OPT can be performed at each local machine. However, if all local machines' data is private, additional communication of $Z_{t}^{(i)}$ is required to compute $D_{t+1}^{(i)}$. Nonetheless, the communication cost associated with transmitting $Z_{t}^{(i)}$ is order-wise comparable to that of sending $Y_{t+1}^{(i)}$ during the aggregation step. Furthermore, the computation of $D_{t+1}^{(i)}$ and the communication of $Y_{t+1}^{(i)}$ only occur when $t+1 \in \mathcal I_T^p$. These factors make the additional communication cost affordable. Additionally, the computation of $D_{t}^{(i)}$ relies on the previous $Z_{t-1}^{(i)}$ and $Z_{t-1}^{(1)}$, which already satisfy DP. Hence, the adaptive composition theorem of DP guarantees that OPT does not introduce additional privacy leakage.
\end{rem}



Algorithm \ref{fedpowerfull} satisfies the strong notion of DP, shown as follows.

\begin{algorithm}[!htbp]
\small

\renewcommand{\algorithmicrequire}{\textbf{Input:}}

\renewcommand\algorithmicensure {\textbf{Output:} }

\caption{FedPower: Full Participation}

\label{fedpowerfull}

\begin{algorithmic}[1]

\STATE \textbf{Input:} distributed dataset $\{A_i\}_{i=1}^m$, iteration rank $r\geq k$, number of iterations $T$, synchronous set $\mathcal I_T^p$, the privacy parameters $(\varepsilon,\delta)$, the variance of Gaussian noise $\nu$.\\
\STATE \textbf{Initialization:} $Z_0^{(i)}=Z_0={\rm orth}(\mathcal N(0,1)^{d\times r})$.\\
\FOR{$t = 0$ to $T-1$}
\STATE  \textbf{Broadcast:} If $t\in \mathcal I_T^p$, the server sends $Z_t$ to every workers and for all $i\in [m]$, let $Z_t^{(i)}=Z_t$;
\STATE \textbf{Local training:}
                                 \STATE \quad\label{local} For all $i\in [m]$, the $i$-th worker independently performs $Y_{t+1}^{(i)}\leftarrow A_iZ_{t}^{(i)}+_{i.i.d.}\mathcal N(0, \nu^2)^{d\times r}$;
                                 \STATE \quad If $t+1\notin \mathcal I_T^p$,  $Z_{t+1}^{(i)}\leftarrow \textsf{orth}(Y_{t+1}^{(i)})$; return to line \ref{local}; $t\leftarrow t+1$;
                                 \STATE \label{computeD}\quad If $t+1\in \mathcal I_T^p$,  move to line \ref{agg1};
\STATE \label{agg1} \textbf{Aggregation:}
                                    \STATE \quad Each worker $i$ sends $Y_{t+1}^{(i)}$ and $Z_{t}^{(i)}$ to the server;
                                     \STATE \quad The server updates $Y_{t+1}^{(i)}\leftarrow Y_{t+1}^{(i)}D_{t+1}^{(i)}$ with $D_{t+1}^{(i)}=\underset{D\in \mathcal F \cap\mathcal O_r }{{\rm argmin}}\ \|Z_{t}^{(i)}D-Z_{t}^{(1)}\|_F$;
                                    \STATE \quad The server performs aggregation: $Y_{t+1}\leftarrow\sum_{i=1}^m {Y_{t+1}^{(i)}}/{m}$ and $Z_{t+1}\leftarrow\textsf{orth}(Y_{t+1})$.\\
\ENDFOR
\STATE \textbf{Output:} orthogonalize the approximated eigen-space:\\
\begin{equation}
\overline{Z}_{T}: = \left\{\begin{array}{cc}
 \sum_{i=1}^m  Z^{(i)}_TD^{(i)}_{T+1}/m, & T\notin \mathcal I_T^p \\\nonumber
 \sum_{i=1}^m Z^{(i)}_T/m, &T\in \mathcal I_T^p
\end{array}\right.
\end{equation}

\end{algorithmic}
\end{algorithm}

\begin{thm}
\label{dpthm}
If we set the noise variance $\nu^2$ sufficiently large such that
$$\nu\geq \frac{2\sqrt{r}m}{n} \max \left\{\sqrt{\frac{T}{\varepsilon}},\, \frac{2\sqrt{2T \log(1/\delta) }}{\varepsilon}\right\},$$
then Algorithm \ref{fedpowerfull} achieves $(\varepsilon, \delta)$-differential privacy towards a honest-but-curious onlooker(or the server).
\end{thm}

We now proceed to provide the utility guarantee for Algorithm \ref{fedpowerfull}.
To that end, we introduce several additional definitions.

\begin{defn}[Data similarity]
\label{assu1}
\[\eta :=\underset{i \in [m]}{\max}\frac{\|A_i-A\|_2}{\|A\|_2},\]
\end{defn}
 This quantity $\eta$ serves as a measure of the distance between the local matrices $A_1, \cdots, A_m$, and the reference matrix $A$. It quantifies the extent to which the local matrices deviate from $A$ in terms of their similarity.
\begin{defn} [Residual error]  \label{def:rho}
Define
\begin{equation}
\phantomsection
\label{3.3}
    \rho_t :=
    {\rm max}_{i \in [m]} \|Z_t^{(i)}D_{t+1}^{(i)} - Z_t^{(1)} \|_2,
\tag{3.3}
\end{equation}
In this equation, if OPT is used, then $D_{t+1}^{(i)}$ is computed via (\ref{3.2}) with $t=t+1$ and $\mathcal F=\mathcal O_r$. On the other hand, if OPT is not used, then $D_{t+1}^{(i)}$ is computed via (\ref{3.2}) with $t=t+1$ and $\mathcal F={\mathbb I_r}$.
\end{defn}

The purpose of $\rho_t$ is to measure the maximum distance between the aligned eigenspace $Z_t^{(i)}D_{t+1}^{(i)}$ and the reference eigenspace $Z_t^{(1)}$ across all local machines.
Its value indicates the extent of misalignment among the local estimators during the aggregation step.
When OPT is used, the value of $\rho_t$ tends to be smaller compared to when OPT is not used. This is because OPT aims to align the local estimators more effectively, reducing the discrepancy between $Z_t^{(i)}D_{t+1}^{(i)}$ and $Z_t^{(1)}$.
When $t \in \mathcal I_T^p$, all local machines reach synchronization, and the local estimators become identical, resulting in $\rho_t = 0$. However, when $t \notin \mathcal I_T^p$, each local update can contribute to enlarging $\rho_t$, leading to increased misalignment among the estimators.
Intuitively, the value of $\rho_t$ is expected to depend on the number of local iterations, denoted by $p$, between two communication rounds. However, our later analysis will show that when OPT is employed, $\rho_t$ does not depend on $p$, while without OPT, it does depend on $p$.
The presence of a residual error, represented by $\rho_t$, is a common phenomenon in previous literature on empirical risk minimization, where local updates are used to improve communication efficiency. Several studies have explored this in the context of distributed optimization \citep{stich2018local, wang2018cooperative, yu2019parallel, li2020convergence, li2019communication, li2021delayed}.

\begin{thm}
\label{utilitythm}
Let $\epsilon':= I_1 + I_2$ with
\begin{equation}
\phantomsection
\label{3.4}
I_1:=\frac{(\sigma_{k}-\sigma_{k+1})^{-1}}{1-(1-1/m){\rm max}_t\rho_{t}}\cdot{\nu }\max\left\{\sqrt{{d \log(dT)}/{m}},\; {\log(dT)}/{m}\right\},
\tag{3.4}
\end{equation}
and
\begin{equation}
\phantomsection
\label{3.5}
I_2:=\frac{(\sigma_{k}-\sigma_{k+1})^{-1}}{1-(1-1/m){\rm max}_t\rho_{t}}\cdot{2}\sigma_1\left(\eta+(2+\eta){\rm max}_t\rho_{t}\right),
\tag{3.5}
\end{equation}
where $\rho_t$ in defined in (\ref{3.3})
and $\nu$ is defined in Theorem \ref{dpthm}.
If ${\epsilon'\lesssim {\rm min}\{\frac{1}{2},\frac{\sqrt{r}-\sqrt{k-1}}{\sqrt{d}}\}}$, then after $T=O(\frac{\sigma_k}{\sigma_k-\sigma_{k+1}}{\rm log}(\frac{d}{\epsilon'}))$ iterations,
the output $\overline{Z}_T$ of Algorithm \ref{fedpowerfull} satisfies
$${\rm sin}\theta_k(\overline{Z}_T,U_k)=\|(\mathbb I_d-\overline{Z}_T\overline{Z}_T^\intercal)U_k\|_2\lesssim \epsilon',$$
with probability at least  $1-(dT)^{-\alpha}-\tau^{-\Omega (r+1-k)}-e^{-\Omega(d)}$, for some positive constants $\alpha$ and $\tau$.
\end{thm}




Theorem \ref{utilitythm} establishes the convergence of Algorithm \ref{fedpowerfull}.
The convergence bound $\epsilon'$ can be divided into two components. $I_1$ is induced by the Gaussian noise required by DP. In particular, we use the tool in \citet{lei2022bias} to bound the sum of the Gaussian random matrices multiplied by orthonormal matrices. $I_1$ decreases as the number of local machines $m$ increases. Apparently, $I_1$ depends on $\sqrt{d}$. However, we note that the denominator $\sigma_{k}-\sigma_{k+1}$ may also depend on $d$. On the other hand, $I_2$ is associated with local computations, which are inevitable in previous research on empirical risk minimization that incorporates local updates to enhance communication efficiency \citep{stich2018local, wang2018cooperative, yu2019parallel, li2020convergence, li2019communication}. In Theorem \ref{thm:rho}, we demonstrate that $\rho_t$ is a function of $\eta$ and becomes sufficiently small when $\eta$ is adequately small. Consequently, this leads to a reduction in the value of $I_2$. Notably, Algorithm \ref{fedpowerfull} simplifies to \textsf{LocalPower}, as introduced in \citet{li2021communication}, when the Gaussian noise is not incorporated, and only $I_2$ remains in the convergence bound.

\begin{thm}
 \label{thm:rho}
Let $\tau(t) \in \mathcal I_T^p$ be the nearest communication time before $t$ and $p=t - \tau(t)$.
Let $\mathrm{e}$ be the natural constant and $\kappa = \|M\|_2\|M^\dagger\|_2$ be the condition number of $M$. Suppose $\eta\leq 1/p$ and $\eta\kappa\leq 1/3$. Then $\rho_t$ is a monotone increasing function of $\eta$. Moreover, when $\epsilon'$ in Theorem \ref{utilitythm} is small enough, we have the following upper bound for $\rho_t$.
\begin{itemize}
	\item With OPT, $\rho_t$ is bounded by
\begin{equation}
\phantomsection
\label{3.6}
{\rm min} \left\{2\mathrm{e}^2\kappa^pp \eta, \frac{\eta\sigma_1}{\delta_k} + 2 \gamma_k^{p/4}  C_t \right\} = O(\eta),
\tag{3.6}
\end{equation}
where $\gamma_k \in (0, 1)$, ${\delta_k \asymp (\sigma_{k}-\sigma_{k+1})}$, and $\limsup_{t} C_t = O(\eta)$.
	\item Without OPT, $\rho_t$ is bounded by
\begin{equation}
\phantomsection
\label{3.7}
4\mathrm{e}\sqrt{k}p \kappa^p \eta=O(\sqrt{k}p \kappa^p \eta).
\tag{3.7}
\end{equation}
\end{itemize}
\end{thm}

Theorem \ref{thm:rho} reveals that when OPT is employed (i.e., $\mathcal F=\mathcal O_r$), $\rho_t = O(\eta)$ without any dependence on $p$. However, in the absence of OPT (i.e., $\mathcal F={\mathbb I_r}$), $\rho_{t} = O(\sqrt{k} p \kappa^p \eta)$ exhibits an exponential dependence on $p$. Theorem \ref{thm:rho} indicates why using OPT has such an exponential improvement on the dependence on $p$. This is mainly because of the property of OPT. Let $O^* = {\rm arg\,min}_{O \in \mathcal O_{r}} \|U - \tilde{U} O\|_F$ for $U, \tilde{U} \in \mathcal O_{d \times r}$.
Then, up to some universal constant, we have
$\|U - \tilde{U} O^*\|_2 \approxeq {\rm dist}(U, \tilde{U}).$
See Lemma~\ref{lem:orth} in  Appendix for a formal statement and detailed proof.
It implies up to a tractable orthonormal transformation, the difference between the orthonormal bases of two subspaces is no larger than the projection distance between the subspaces.
By the Davis-Kahan theorem (see Lemma~\ref{lem:DK}), their projection distance is not larger than $O(\eta)$ up to some problem-dependent constants.
However, without OPT, we have to use perturbation theory to bound $\rho_t$, which inevitably results in exponential dependence on $p$ (see Lemma \ref{lem:rho2}).


\subsection{Federated Power Method under Partial Participation Protocol}
\label{subsec:partial}

To address the limitation of the full participation protocol in accounting for stragglers, we introduce a partial participation protocol. In this protocol, during each communication round, the server collects the outputs from the first $K$ (where $K \leq m$) local machines that respond. Once the server has collected $K$ outputs, it stops waiting for the remaining machines. The local machines that do not respond within the given round are referred to as "stragglers" for that particular iteration. We denote $\mathcal{S}_t$ (where $|\mathcal{S}_t|=K$) as the set of indices corresponding to the local machines participating in the $t$-th iteration (when $t \in \mathcal{I}_T^p$). To model this behavior, we employ the following sampling and aggregation schemes, which have also been utilized in prior works such as \citet{li2020federated}.

\begin{sche}
\label{sch1}
The server generates $\mathcal S_t$ by {i.i.d.} sampling with replacement from $\{1,\dots,m\}$ for $K$ times. Specifically, index $i$ is selected with probability $\frac{1}{m}$, and the elements in $\mathcal S_t$ may occur more than once.  The server then aggregates according to $$Y_t= \frac{1}{K}\underset{i\in \mathcal S_t}\sum Y_t^{(i)}.$$
\end{sche}

Such an aggregation policy ensures that the partial participation protocol agrees with the full participation protocol in expectation. Indeed, considering only the randomness that comes from $\mathcal S_t$, we observe that
$$\mathbb E_{\mathcal S_t}(Y_t)=\frac{1}{K}\mathbb E_{\mathcal S_t}(\sum_{k=1}^K Y_t^{(i_k)})=\mathbb E_{\mathcal S_t}Y_t^{(i_1)}=\frac{1}{m}\sum_{i=1}^m Y_t^{(i)}.$$
It is worth noting that the proof does not strictly require the expectation of the partial aggregation to be identical to that of the full aggregation.

The proposed \textsf{FedPower} under the partial participation protocol is summarized in Algorithm \ref{fedpowerpartial}.
In this algorithm, we use $\tau(t)$ to denote the latest synchronization step that occurs before iteration $t$.
\begin{rem}
For simplicity, we assume that the first local machine is always active and keeps updating the parameters no matter whether it is selected or not. Alternatively, one can choose any one of the active machines in each iteration as the reference machine.
\end{rem}


Similar to the proof of Theorem \ref{dpthm}, we can easily demonstrate that Algorithm \ref{fedpowerpartial} achieves $(\varepsilon, \delta)$-DP against an honest-but-curious onlooker for the same requirement of $\nu$.
The following theorem provides the convergence bound for Algorithm \ref{fedpowerpartial}.

\begin{algorithm}[t!]
\small

\renewcommand{\algorithmicrequire}{\textbf{Input:}}

\renewcommand\algorithmicensure {\textbf{Output:} }

\caption{FedPower: Partial Participation}

\label{fedpowerpartial}

\begin{algorithmic}[1]

\STATE \textbf{Input:} distributed dataset $\{A_i\}_{i=1}^m$, iteration rank $r\geq k$, number of iterations $T$, synchronous set $\mathcal I_T^p$, the number of participated machines $K$, the privacy parameters $(\varepsilon,\delta)$, the variance of Gaussian noise $\nu$.\\
\STATE \textbf{Initialization:} $Z_0^{(i)}=Z_0={\rm orth}(\mathcal N(0,1)^{d\times r})$.\\
\FOR{$t = 0$ to $T-1$}
\STATE  \textbf{Broadcast:} If $t\in \mathcal I_T^p$, the server sends $Z_t$ to every workers and for all $i\in [m]$, let $Z_t^{(i)}=Z_t$;
\STATE \textbf{Local training:}
                                 \STATE \quad\label{local} For all $i\in [m]$, the $i$-th worker independently performs $Y_{t+1}^{(i)}\leftarrow A_iZ_{t}^{(i)}+_{i.i.d.}\mathcal N(0, \nu^2)^{d\times r}$;
                                 \STATE \quad If $t+1\notin \mathcal I_T^p$,  $Z_{t+1}^{(i)}\leftarrow \textsf{orth}(Y_{t+1}^{(i)})$; return to line \ref{local}; $t\leftarrow t+1$;
                                 \STATE\quad If $t+1\in \mathcal I_T^p$; move to line \ref{agg};
\STATE \label{agg}  \textbf{Sampling and Aggregation:}
\STATE \quad The server generates $\mathcal S_{t+1}$ with cardinality $K$ according to the sampling with replacement strategy;
\STATE \quad For $i\in \mathcal S_{t+1} \cup \{1\}$, each worker $i$ sends $Y_{t+1}^{(i)}$ and $Z_{t}^{(i)}$ to the server;
\STATE \quad For $i\in \mathcal S_{t+1} \cup \{1\}$, the server updates $Y_{t+1}^{(i)}\leftarrow Y_{t+1}^{(i)}D_{t+1}^{(i)}$ with $D_{t+1}^{(i)}=\underset{D\in \mathcal F \cap\mathcal O_r }{{\rm argmin}}\ \|Z_{t}^{(i)}D-Z_{t}^{(1)}\|_F$;
\STATE \quad The server performs partial aggregation: $Y_{t+1}\leftarrow\sum_{i\in \mathcal S_{t+1}} {Y_{t+1}^{(i)}}/{K}$ and $Z_{t+1}\leftarrow\textsf{orth}(Y_{t+1})$;\\
\ENDFOR
\STATE \textbf{Output:} orthogonalize the approximated eigenspace:\\
\begin{equation}
\overline{Z}_{T}: = \left\{\begin{array}{cc}
 \sum_{i\in \mathcal S_{\tau(T)}} Z^{(i)}_TD^{(i)}_{T+1}/K, & T\notin \mathcal I_T^p \\\nonumber
 \sum_{i\in \mathcal S_{\tau(T)}} Z^{(i)}_T/K, &T\in \mathcal I_T^p
\end{array}\right.
\end{equation}
\end{algorithmic}
\end{algorithm}


\begin{thm}
\label{utilitythmpartial}
Let $\epsilon'':= I_1 + I_2 + I_3$ with
\begin{equation}
\phantomsection
\label{3.8}
I_1:=\frac{(\sigma_{k}-\sigma_{k+1})^{-1}}{1-{\rm max}_t\rho_{t}}\cdot {\nu }\max\left\{\sqrt{\frac{d \log(dT)}{K}},\; \frac{\log(dT)}{K}\right\},
\tag{3.8}
\end{equation}
\begin{equation}
\phantomsection
\label{3.9}
I_2:=\frac{(\sigma_{k}-\sigma_{k+1})^{-1}}{1-{\rm max}_t\rho_{t}}\cdot{2}\sigma_1\left(\eta+(2+\eta){\rm max}_t\rho_{t}\right),
\tag{3.9}
\end{equation}
and
\begin{equation}
\phantomsection
\label{3.10}
I_3:=\frac{(\sigma_{k}-\sigma_{k+1})^{-1}}{1-{\rm max}_t\rho_{t}}\cdot \sigma_1\max\left\{\sqrt{\frac{\log(dT)}{K}},\; \frac{\log(dT)}{K}\right\},
\tag{3.10}
\end{equation}
where $\rho_t$ in defined in (\ref{3.3})
and $\nu$ is defined in Theorem \ref{dpthm}.
If ${\epsilon''\lesssim {\rm min}\{\frac{1}{2},\frac{\sqrt{r}-\sqrt{k-1}}{\sqrt{d}}\}}$, then after $T=O(\frac{\sigma_k}{\sigma_k-\sigma_{k+1}}{\rm log}(\frac{d}{\epsilon''}))$ iterations,
the output $\overline{Z}_T$ of Algorithm \ref{fedpowerpartial} satisfies
$${\rm sin}\theta_k(\overline{Z}_T,U_k)=\|(\mathbb I_d-\overline{Z}_T\overline{Z}_T^\intercal)U_k\|_2\lesssim \epsilon',$$
with probability at least $1-(dT)^{-\beta}-\tau^{-\Omega (r+1-k)}-e^{-\Omega(d)}$ for some positive constants $\beta$ and $\tau$.
\end{thm}

Unlike Theorem \ref{utilitythm}, the convergence bound of Algorithm \ref{fedpowerpartial} in Theorem~\ref{utilitythmpartial} is divided into three components. As before, $I_1$ arises from the Gaussian noise. Notably, in contrast to the full participation protocol, we observe that $I_1$ depends on the number of active machines $K$ rather than the total number of local machines $m$. $I_2$ is incurred by the local iterates, which has the same effects as in the full participation protocol. $I_3$ can be regarded as the bias that the sampling brings, where a larger $K$ yields a smaller $I_3$. In addition, $\rho_t$ can be upper bounded differently depending on whether OPT is used (see Theorem \ref{thm:rho}).


\subsection{Discussion}
\label{subsec:dis}

\paragraph{Bound for $\eta$.}

The convergence of \textsf{FedPower} depends on the smallness of $\eta$, which indicates that each local data set $A_i$ is a typical representative of the whole data matrix $A$. Previous work \citep{gittens2016revisiting,woodruff2014sketching,wang2016spsd} has demonstrated that uniform sampling and partition sizes described in Lemma~\ref{lem:uniform} are sufficient for $A_i$ to provide a good approximation of $A$. It would be interesting to explore whether the requirement of small $\eta$ can be eliminated, similar to the work of \citet{karimireddy2020scaffold} in the context of regression, by using gradient-based methods \citep{ammad2019federated,chai2020secure,chen2021distributed} for eigen-decomposition.

\paragraph{Effect of $p$.}

Theorems \ref{utilitythm} and \ref{utilitythmpartial} demonstrate that for a given error tolerance $\epsilon$, the number of required communications is approximately $\lfloor T/p\rfloor$, where $T=O(\frac{\sigma_k}{\sigma_{k+1}}{\rm log}(\frac{d}{\epsilon}))$. Thus, increasing the number of local iterations (larger $p$) enhances communication efficiency. However, this advantage of larger $p$ comes at a cost. As shown in Theorem \ref{thm:rho}, the residual error $\rho_t$ (with OPT) is bounded by ${\rm min} \left\{a_1, a_2\right\}$, where $a_1=2\mathrm{e}^2\kappa^pp \eta$ and $a_2=\frac{\eta\sigma_1}{\delta_k} + 2 \gamma_k^{p/4} C_t$. Here, $\gamma_k \in (0, 1)$ and $\limsup_{t} C_t = O(\eta)$. For moderate values of $p$, $a_2$ is smaller than $a_1$ and is decreasing in $p$, indicating that larger $p$ leads to a smaller final error. However, it is important to note that its limit (i.e., $\frac{\eta\sigma_1}{\delta_k}$) does not depend on $p$ when the final error $\epsilon',\epsilon''$ is sufficiently small.

\paragraph{Decay $p$ gradually.}
We have observed that using a large value of $p$ accelerates the initial convergence but leads to a larger final error. On the other hand, setting $p = 1$ achieves the lowest error, although it also has the slowest convergence rate. Similar observations have been made in the context of distributed empirical risk minimization \citep{wang2019scalable,li2019communication}. To strike a balance between fast initial convergence and achieving a vanishing final error, we propose gradually decaying the value of $p$. In particular, we set
\begin{equation}
\phantomsection
\label{3.11}
\mathcal I_T^{p, \text{decay}} = \bigg\{ t \in [T]:  t = \sum_{i=0}^l \max(p-i, 1), \ l \geq 0 \bigg\}.
\tag{3.11}
\end{equation}

\paragraph{Dependence on $\sigma_{k} {-}\sigma_{k+1}$.}

Our result depends on the difference $\sigma_k - \sigma_{k+1}$ even when $r > k$, where $r$ represents the number of columns used in the subspace iteration. By leveraging the technique introduced by \citet{balcan2016improved} instead of \citet{hardt2014noisy}, we might be able to further improve the result to achieve a slightly milder dependency on $\sigma_k - \sigma_{q+1}$, where $q$ is an intermediate integer between $k$ and $r$. It is worth noting that certain work (such as \citet{musco2015randomized}) have obtained gap-free results, but their goal is to find a good subspace that captures nearly as much variance as the top eigenvectors of $A$. In contrast, our focus is on quantifying the deviation of estimated eigenvectors from the true eigenvectors of $A$.

\section{Related Works and Contributions}
\label{sec:related}

Partial eigen-decomposition or PCA is one of the most important and popular techniques in modern statistics and machine learning.
A multitude of researches focus on iterative algorithms such as power iterations or its variants \citep{Golub2012matrix,saad2011numerical}.
These deterministic algorithms inevitably depend on the spectral gap, which can be quite large in large-scale problems.
Another branch of algorithm seek alternatives in stochastic and incremental algorithms \citep{oja1985stochastic,arora2013stochastic,shamir2015stochastic,shamir2016convergence,de2018accelerated}.
Some work could achieve gap-free convergence rate and low-iteration-complexity \citep{musco2015randomized,shamir2016convergence,allen2016lazysvd}. Other work seeks to accelerate the eigen-decomposition via randomization \citep{halko2011finding,witten2015randomized,guo2020randomized,zhang2022randomized}.

Large-scale problems and large decentralized datasets necessitate cooperation among multiple worker nodes to overcome the obstacles of data storage and heavy computation. A review of distributed algorithms for PCA can be found in \citet{wu2018review}. One line of work employs divide-and-conquer algorithms that require only one round of communication \citep{garber2017communication,fan2019distributed,bhaskara2019distributed,charisopoulos2021communication}. However, such algorithms often require large local datasets to achieve a certain level of accuracy. Another line of work uses iterative algorithms for distributed eigenspace estimation, which involve multiple communication rounds. These algorithms typically require a much smaller sample size and can achieve arbitrary accuracy. Some work utilizes the shift-and-invert framework for PCA, which transforms the problem of computing the leading eigenvector into the approximate solution of a small linear system of equations \citep{garber2015fast,garber2016faster,allen2016lazysvd,garber2017communication,gang2019fast,chen2021distributed}.

The technique of local updates has proven to be a simple yet powerful tool in distributed empirical risk minimization \citep{mcmahan2017communication,zhou2017convergence,stich2018local,wang2018cooperative,yu2019parallel,li2020convergence,li2019communication,khaled2019first}. However, it is important to note that our analysis of \textsf{FedPower} significantly differs from the local SGD algorithms commonly used in empirical risk minimization \citep{zhou2017convergence,stich2018local,wang2018cooperative,yu2019parallel,li2020convergence,li2019communication,khaled2019first}.
The main challenge in analyzing \textsf{FedPower} arises from the fact that local SGD algorithms for empirical risk minimization often involve explicit forms of (stochastic) gradients. However, in the case of SVD or PCA, the gradient cannot be explicitly expressed, making existing techniques inapplicable.

In terms of privacy preservation, several privacy-preserving algorithms for PCA or eigen-decomposition have been proposed in the single-machine setting \citep{chaudhuri2012near,hardt2013beyond,dwork2014analyze,hardt2014noisy,upadhyay2018price,amin2019differentially,singhal2021privately,dong2022differentially}. In the distributed setting, \citet{ge2018minimax} introduced a power-method-based privacy-preserving distributed sparse PCA algorithm. However, the communication cost of that algorithm is higher compared to \textsf{FedPower}. \citet{gra2020federated} proposed a federated, asynchronous, and input-perturbed differentially private algorithm for decentralized PCA, where the theoretical justification of the estimated eigen-space is lacked. \citet{chai2022practical} proposed a federated SVD method using a lossless matrix masking scheme, which is weaker than differential privacy. There are also works on related matrix factorization problems, such as \citet{ammad2019federated,chai2020secure}.


\paragraph{Our Contributions.}
The main contributions of this work can be summarized as follows.
\begin{itemize}
\item First, we develop a set of algorithms called \textsf{FedPower}, which is communication efficient, privacy-preserving, and adapts to  stragglers.
\item Second, we propose to decay the communication interval $p$, over time. In this way, the loss drops fast in the beginning and converges to the optimal solution in the end. We use OPT to post-processes the output matrices of the $m$ nodes after each iteration so that the $m$ nodes are close to each other.
\item Last but not least, we provide a rigorous analysis of the privacy bound by using the tools of RDP. We provide a framework to analyze the convergence bound of \textsf{FedPower}. We make a delicate analysis of the error influenced by local iterates, DP's perturbation, and random straggling of local machines.
\end{itemize}
Compared with the shorter version of this paper \citep{li2021communication}, the novelty of this work lie in the following aspects.
\begin{itemize}
\item Besides the communication efficiency considered in \citet{li2021communication}, we also consider the privacy breaches and stragglers' effect which are real issues faced by distributed eigenspace estimation.
\item We analyze the convergence of the virtual sequence in the form of $\overline{Z}_t$ instead of $\overline{Y}_t$  in \citet{li2021communication}, which improves the convergence bound by a factor of $\kappa$.
\item We utilize the advanced tools of random matrix theory to bound the sum of OPT-transformed Gaussian matrices that comes from DP perturbation and the sampling bias that comes from a random sampling of local machines.
\end{itemize}

\section{Experiments}
\label{sec:experiment}
In this section, we numerically evaluate the efficacy of the proposed algorithm \textsf{FedPower}. Specifically, we generate the data from two low-rank statistical models, namely, the spiked covariance model \citep{cai2015optimal} and the stochastic block model (SBM) \citep{lei2022bias}.


\paragraph{Model I (Spiked covariance model)}
We choose $n=2\times 10^6$, $d=100$, $m=20$, $k=4$, $r=5$ in the experiment. For each local machine, the data samples $M_i\in \mathbb R^{\frac{n}{m}\times d}$ are generated independently from a multivariate Gaussian distribution with population covariance $\Sigma$ and mean $0$ with
$$\Sigma = U_{d\times k}U^\intercal_{k\times d}+\sigma^2 \mathbb I_d,$$
where $\sigma=0.6$ and $U$ is obtained by orthogonalizing a random $d\times k$ matrix with i.i.d. entries from $N(0.5, 1)$. Each sample is then normalized. The whole data matrix $A$ is formulated using (\ref{2.2}) and (\ref{2.3}).

\paragraph{Model II (Stochastic block model)}
We choose $n=2\times 10^4$, $d=1000$, $m=20$, $k=r=2$. It is easy to check that $\frac{n}{m}=d$. For each $l\in \{1,...,m\}$, the $l$th local machine owns a $1000\times 1000$ network adjacency matrix $M_l$ from SBM. Assume the network nodes in $m$ local machines are aligned and they are consistently assigned to $k$ non-overlapping communities with the community assignment of nodes $l$ denoted by $g_l \in\{1,...,k\}$. Given $g_i$'s, each entry of $M_i$ is generated independently according to $$[M_l]_{ij}\sim {\rm Bernoulli}(B_{g_ig_j}),\quad i<j,$$
$[M_l]_{ji}=[M_l]_{ij}$ and $[M_l]_{ii}=0$. $B$ is the connectivity matrix. For the first half networks, $B^{(1)}=0.8W$ and for the second half networks, $B^{(2)}=0.6W$, where $W=[0.25,0.1;0.1,0.25]$. Each row of $M_i$'s is then normalized and the whole data matrix $A$ is formulated using (\ref{2.2}) and (\ref{2.3}). The rationality of squaring the adjacency matrix can be found in \citep{lei2022bias}.

\paragraph{Experimental setup}
Under Model I and Model II, we study the effect of local iterates $p$, the privacy budget $\varepsilon$, and the proportion of participated local machines $\frac{K}{m}$ (in the partial participation scheme). The number of iterations is fixed to be $T=10$. And we generally choose the number local iterates $p=2$, the total privacy budget $\varepsilon=0.5$, the other privacy parameter $\delta=10^{-4}$, and the number of participated machines $K=m$, which may vary when we study the effect of $p$, $\varepsilon$ and $\frac{K}{m}$. We use the decay $p$ strategy in all the experiments. We study how the projection distance between the eigenvector computed by \textsf{FedPower} and the eigenvector of $A$ varies with the number of communications.

\paragraph{Effect of the number of local iterations $p$:}
Figures \ref{spiked}(a) and \ref{sbm}(a) present the results for $p=1,2,4$ under Model I and Model II, respectively. It was observed that larger values of $p$ led to smaller projection distances compared to the baseline case of $p=1$. This demonstrates the communication efficiency of \textsf{FedPower}, as it achieved lower projection distances with the same number of communications. We here also highlight the requirements for obtaining these results, including the small noise from differential privacy, which is influenced by the number of iterations $T$ and the sensitivity, as well as a small number of local machines to reduce the aggregation bias.


\paragraph{Effect of the number of privacy budget $\varepsilon$.} Figure \ref{spiked}(b) and \ref{sbm}(b) show the results with $\varepsilon=10,1,0.5$ under Model I and Model II, respectively. As expected, increasing the privacy budget $\varepsilon$ leads to improved performance of \textsf{FedPower}, indicating the trade-off between privacy and accuracy. It should be noted that we fix the other privacy parameter $\delta$ to be $10^{-4}$, which ensures that $\delta$ is smaller than the number of samples per-machine and is a desirable choice in DP.

\paragraph{Effect of the proportion of participated machines $\frac{K}{m}$.}
In addition to the full participation protocol, we investigate the impact of the proportion of participated machines $\frac{K}{m}$ in the partial participation scheme. Figures \ref{spiked}(c) and \ref{sbm}(c) demonstrate the results for $\frac{K}{m}=1,0.6,0.4$ under Model I and Model II, respectively. Notably, larger values of $\frac{K}{m}$ tend to yield better results compared to smaller values. However, the overall influence of $\frac{K}{m}$ on the algorithm's accuracy is not significant, thus confirming the effectiveness of \textsf{FedPower} in real-world scenarios with a notable proportion of stragglers.


\begin{figure}[htb]{}
\centering
\subfigure[Effect of $p$]{\includegraphics[height=4.6cm,width=5cm,angle=0]{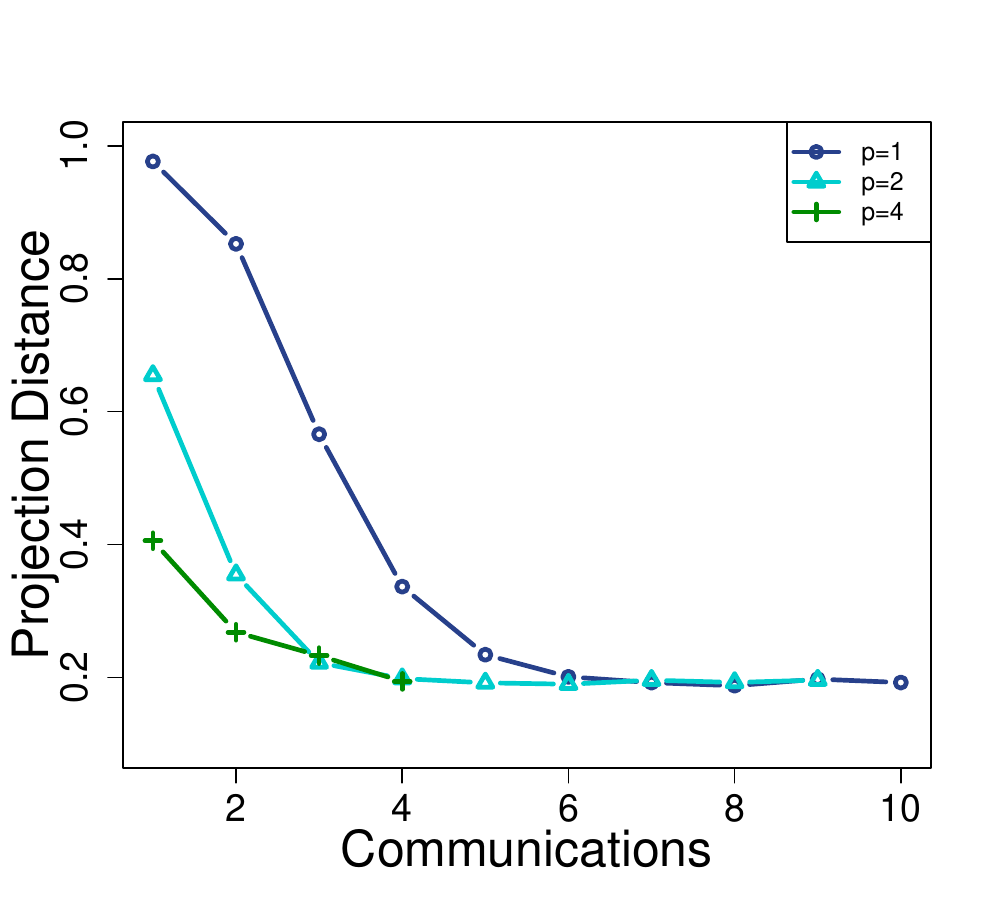}}
\subfigure[Effect of $\varepsilon$]{\includegraphics[height=4.6cm,width=5cm,angle=0]{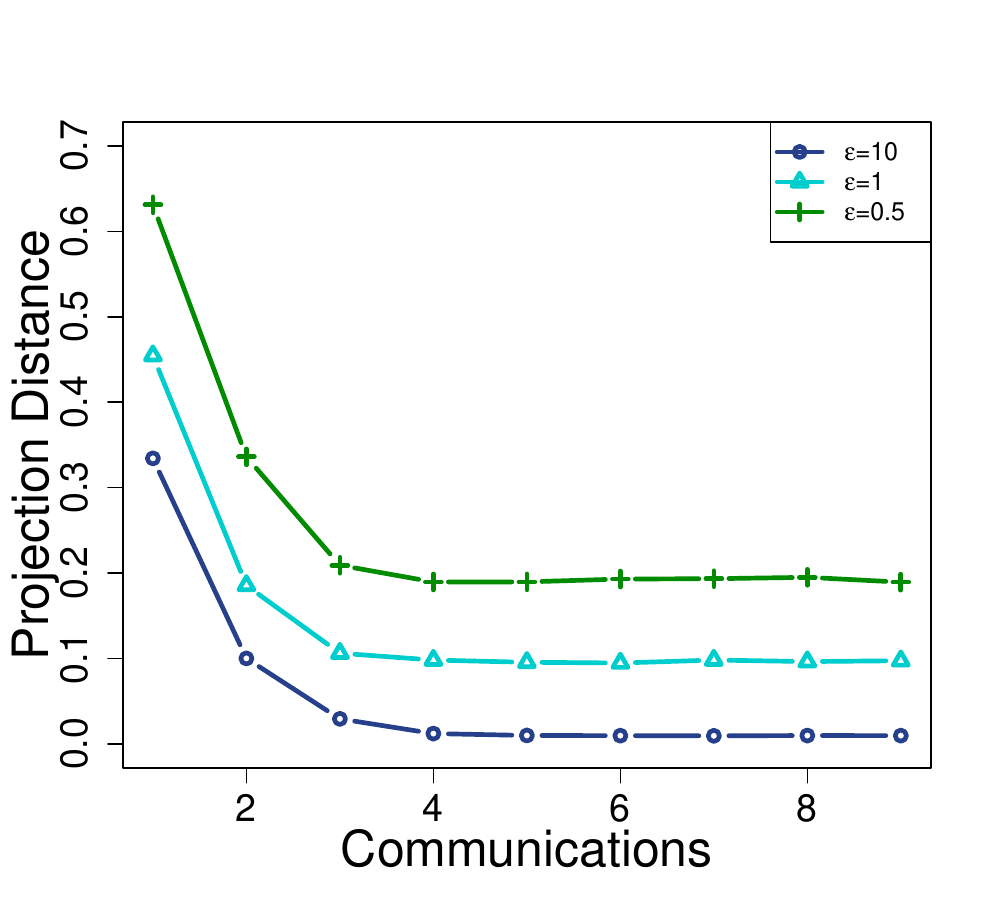}}
\subfigure[Effect of $\frac{K}{m}$]{\includegraphics[height=4.6cm,width=5cm,angle=0]{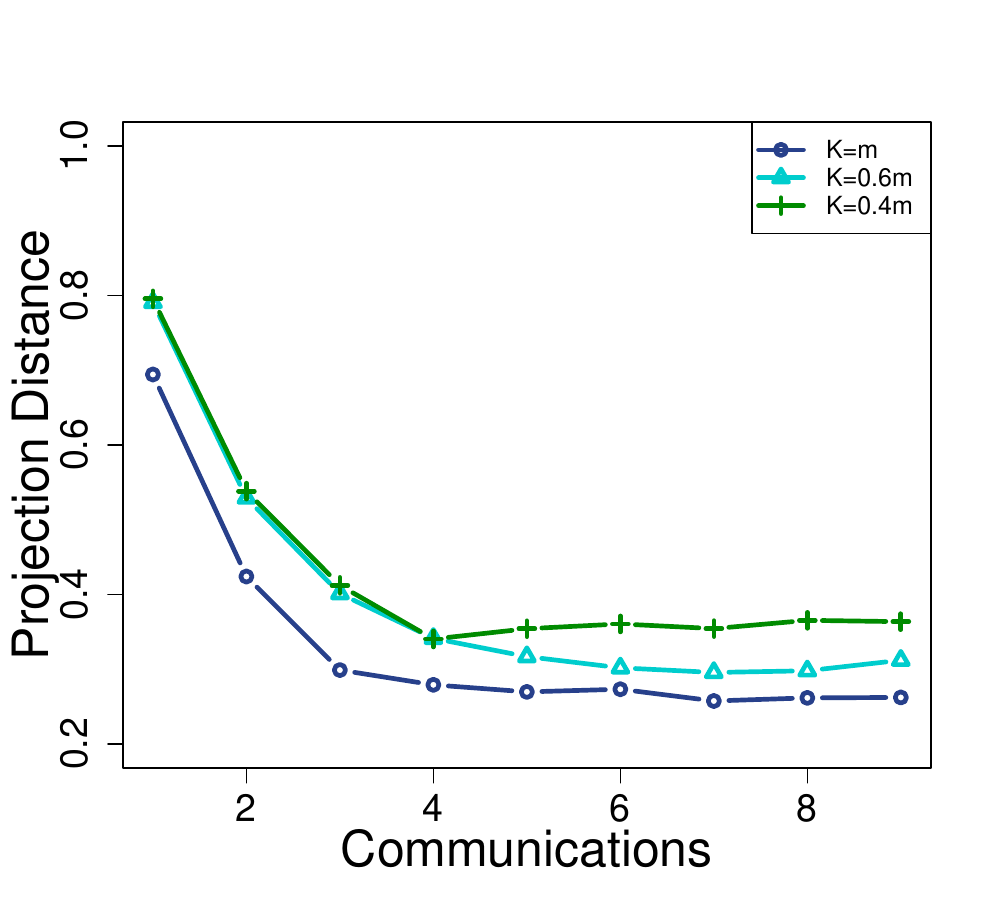}}\\
\caption{ The numerical performance of \textsf{FedPower} under Model I. We vary $p$, $\varepsilon$, and $\frac{K}{m}$ in (a), (b), and (c) respectively to show their effects on the performance of \textsf{FedPower}.  }\label{spiked}
\end{figure}

\begin{figure}[htb]{}
\centering
\subfigure[Effect of $p$]{\includegraphics[height=4.6cm,width=5cm,angle=0]{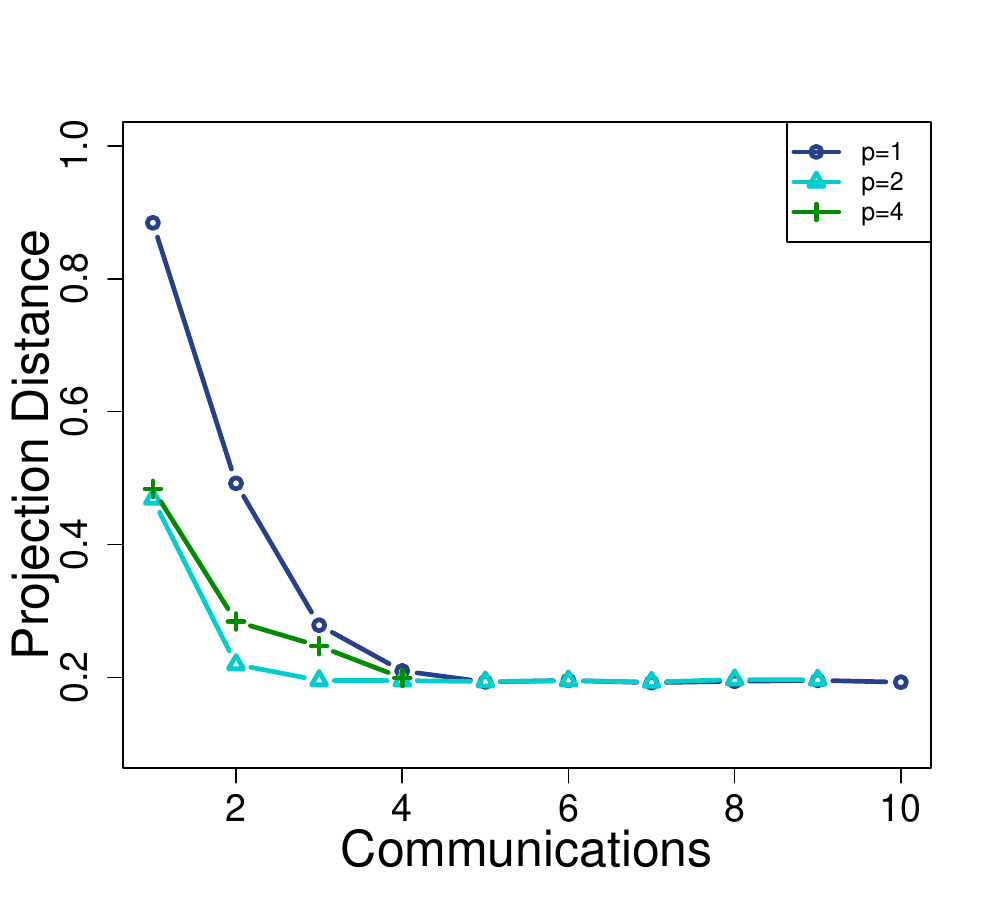}}
\subfigure[Effect of $\varepsilon$]{\includegraphics[height=4.6cm,width=5cm,angle=0]{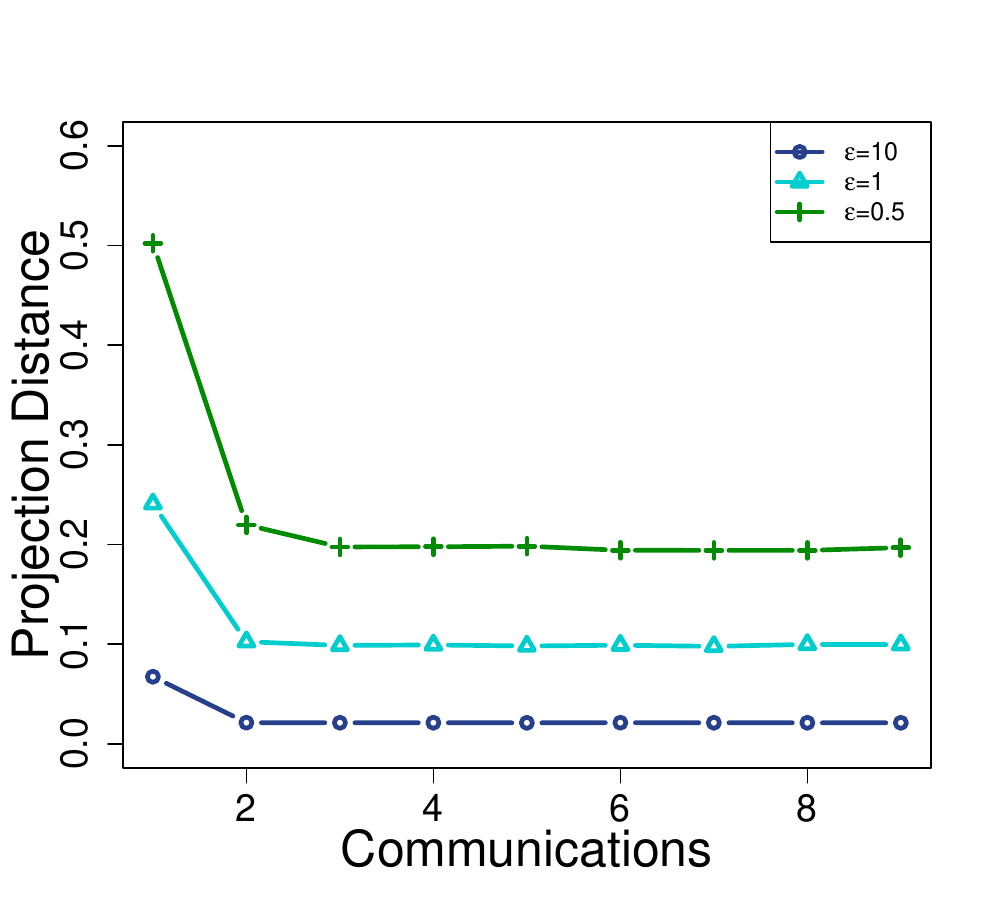}}
\subfigure[Effect of $\frac{K}{m}$]{\includegraphics[height=4.6cm,width=5cm,angle=0]{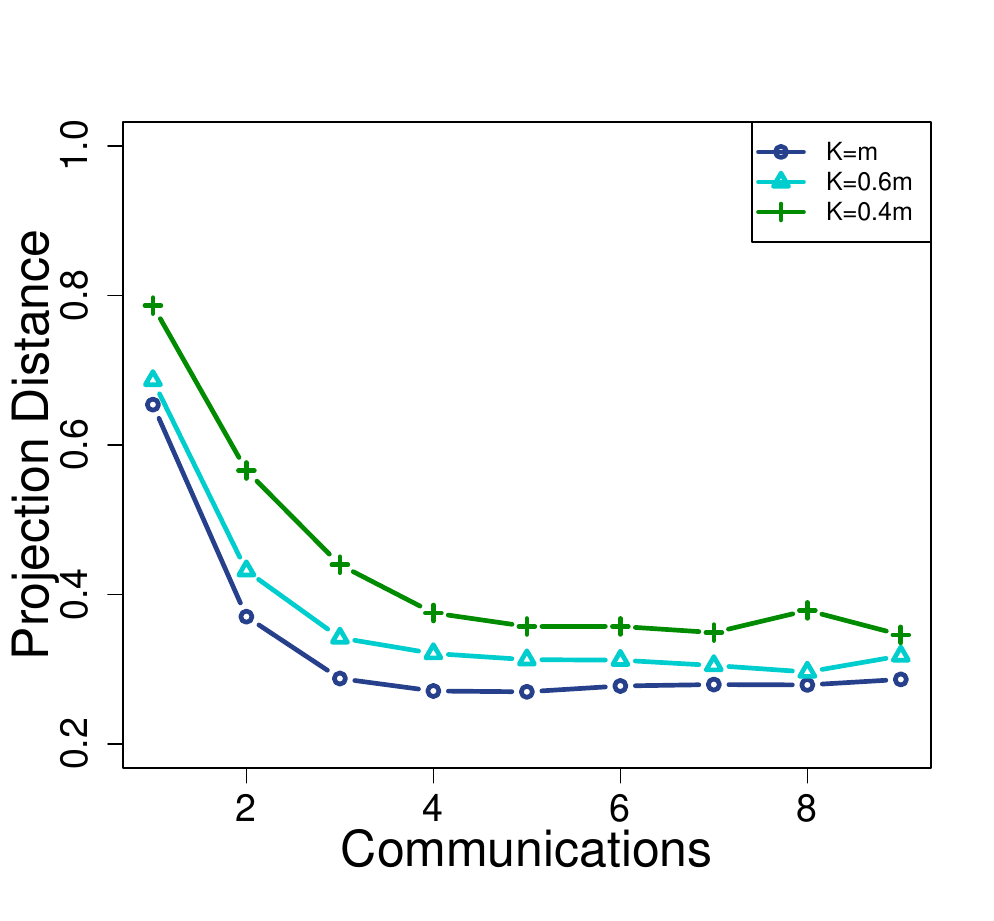}}\\
\caption{ The numerical performance of \textsf{FedPower} under Model II. We vary $p$, $\varepsilon$, and $\frac{K}{m}$ in (a), (b), and (c) respectively to show their effects on the performance of \textsf{FedPower}. }\label{sbm}
\end{figure}

\newpage
\section{Conclusion and Discussion}
\label{sec:conclusion}
We present \textsf{FedPower}, a distributed algorithm designed for the federated learning framework to address the eigenspace estimation problem while ensuring communication efficiency, privacy preservation, and tolerance to stragglers.
Our approach involves performing multiple noisy local power iterations between consecutive communications on each worker machine, where the noise comes from the DP's requirement. Full or partial aggregation scheme is performed after every $p$ iterates. Methodologically, our algorithms provide a flexible and general framework for the computation of leading eigenspace in the modern machine learning setting. Theoretically, our analysis gives a new application of the noisy power method \citep{hardt2014noisy} by combining the perturbed iterate analysis. In addition, the proposed algorithms can be applied to a wide range of statistical and machine learning tasks, including matrix completion, clustering, and ranking, among others.


This work is a first step towards the power-method-based leading eigenspace estimation in the realistic federated learning regime.
There are several challenges and open problems that require further investigation.
First, the coupled factors $Z_t^{(i)}$ and $R_t^{(i)}$ in $W_t:=\sum_i W_t^{(i)}$ (see, e.g., (\ref{A.5})) make $\|W_t\|_2$ challenging to bound. In this work, we bound $\|W_t\|_2$ by taking maximization over all $i\in[m]$ and analyzing each $\|W_t^{(i)}\|_2$ respectively. However, ideally, one may hope that the summation is left inside the spectral norm and the bound of $\|W_t\|_2$ depends on the overall error $\|\sum _i A_i-A\|_2$ like the one-shot framework of \citet{charisopoulos2021communication}. In that case, the theoretical dependence on the heterogeneity parameter $\eta$ can also be relaxed. Therefore, new analyzing tools are in urgent need.

Second, it is well-known that privacy can be amplified by subsampling, that is, a DP mechanism run on a random subsample of a dataset provides higher privacy guarantees than when run on the entire dataset. Most literature considers the \emph{sampling without replacement} strategy, including Possion subsampling and uniform subsampling. The privacy bound of the {sampling without replacement} strategy is easy to analyze. Due to technical challenges, very few literature consider the \emph{sampling with replacement} strategy. In particular, \citet{balle2018privacy} provides a tight analysis of the privacy bound of the Gaussian mechanism by using the tools of \emph{privacy profiles}. In our context, the subsampling of local machines ideally can lead to privacy amplification by a factor $K/m$. However, since our algorithm is not additive Gaussian noise perturbation, the privacy profile is difficult to compute. Thus, the tools of sampling with replacement in \citet{balle2018privacy} are not applicable. Therefore, in terms of privacy analysis, the privacy amplification of the sampling with replacement is harder than that of the sampling without replacement. While in terms of the utility analysis, the non-dependency of local machines in the sampling without replacement strategy would make the analysis challenging. Therefore, whatever the sampling strategy is, the privacy analysis and utility analysis can not be satisfied simultaneously. New concentration inequality on the sum of uniformly sampled (without replacement) matrices and new privacy amplification bound for sampling with replacement strategy without using the privacy profile is imperative.

Third, we mainly consider the curious-but-honest attacker who can see the parameters via communication. In some cases, there also exists an \emph{external adversary} who only can see the final published results with additional prior information. For example, people who participate in the data collection and model design may be such kind of attackers. Hence, it is beneficial to see how much privacy the \textsf{FedPower} leaks to the external adversary. For the same noise, the privacy leakage (namely, the privacy budget $\varepsilon$) to the external adversary is hopefully smaller than that to the curious onlooker. However, in our context, the DP's noise to the external adversary is the average of \emph{OPT transformed} Gaussian random matrices, which makes it difficult to quantify the corresponding privacy leakage. This problem deserves further study.


\section*{Appendix}

Subsection \textbf{A} and \textbf{B} includes the proofs (also the proof sketch if necessary) corresponding to the full and partial participation protocols, respectively. Subsection \textbf{C} contains the technical lemmas. Subsection \textbf{D} presents auxiliary lemmas used in the proofs. Subsection \textbf{E} introduces the formal definitions and lemmas on metrics between two
subspaces.
\subsection*{A. Full participation}
\subsubsection*{Proof of Theorem \ref{dpthm}}
We provide the proof in the following steps.
\paragraph{First step: Bounding the sensitivity.}
Without loss of generality, consider the $i$-th local machine in $t$-th iteration. Let $[Z_{t-1}^{(i)}]_{\cdot l}$ denote the $l$-th ($1\leq l\leq r$) column of $Z_{t-1}^{(i)}$ and let $A_i'= \frac{m}{n}M_i'^\intercal M_i'$ be the neighboring dataset of $A_i= \frac{m}{n}M_i^\intercal M_i$, where $M_i$ and $M_i'$ differ in one row (denoted by $\omega$ and $\omega'$) and all their rows are normalized to have Euclidean norm 1. Conditioning on $[Z_{t-1}^{(i)}]_{\cdot l}$, we have that the $l_2$-sensitivity of $A_i[Z_{t-1}^{(i)}]_{\cdot l}$ is
\begin{align*}
\|A_i [Z_{t-1}^{(i)}]_{\cdot l}- A_i' [Z_{t-1}^{(i)}]_{\cdot l}\|_2&\leq \|A_i -A_i' \|_2\cdot \|[Z_{t-1}^{(i)}]_{\cdot l}\|_2 \leq \|A_i-A_i'\|_{\tiny F}\nonumber\\
& =\frac{m}{n}\|\omega\omega^\intercal-\omega'\omega'^\intercal\|_{\tiny F}\leq \frac{m}{n} (\|\omega\omega^\intercal\|_{\tiny F}+\|\omega'\omega'^\intercal\|_{\tiny F})\nonumber\\
&=\frac{m}{n} (\|\omega\|_2\|\omega\|_2+\|\omega'\|_2\|\omega'\|_2)=\frac{2m}{n}.
\end{align*}
Stacking $r$ such vectors together to obtain a $d\times r$-dimensional vector, we have that the total sensitivity of $A_iZ_{t-1}^{(i)}$ is bounded by $\sqrt{r}\cdot\frac{2m}{n}$.

\paragraph{Second step: Tight composition of RDP.}
Suppose that we add i.i.d. Gaussian noise to each entry of $A_iZ_{t-1}^{(i)}$ according to
\begin{equation*}
{Y_t^{(i)}=A_iZ_{t-1}^{(i)}+\mathcal N(0, (\frac{2\sqrt{r}m}{n}\sigma_g)^2)^{{d\times r}},}
\end{equation*}
then by Proposition \ref{gaussrdp}, conditioning on $Z_{t-1}^{(i)}$, we have that the above mechanism satisfies $(\alpha, \frac{\alpha}{2\sigma_g^2})$-RDP.

After $p$ local iterations and $\lfloor T/p\rfloor$ communication rounds, by invoking the adaptive composition theorem of RDP in Proposition \ref{composition}, we have that Algorithm \ref{fedpowerfull} satisfies $(\alpha, \frac{T\alpha}{2\sigma_g^2})$-RDP.

\paragraph{Third step: Optimizing $\alpha$ and choosing $\sigma_g$.}
To be interpretable, we convert RDP to DP via Proposition \ref{rdptodp}, that is, Algorithm \ref{fedpowerfull} satisfies
$(\frac{T\alpha}{2\sigma_g^2}+\frac{\log (1/\delta)}{\alpha-1},\delta)$-RDP
for any $\alpha>1$ and $\delta>0$. To select an optimal $\alpha$, we consider the following optimization problem,
\begin{equation}
\label{A.1}
\underset{\alpha>1}{\min} \;\varepsilon (\alpha):= \frac{T\alpha}{2\sigma_g^2}+\frac{\log (1/\delta)}{\alpha-1}.
\tag{A.1}
\end{equation}
It is easy to obtain that the optimal $\alpha$ for \eqref{A.1} is $$\alpha_{\rm opt}:=1+\sqrt{2\sigma_g^2\log(1/\delta)/T},$$
and the corresponding $\varepsilon (\alpha_{\rm opt})$ is
\begin{equation}
\label{A.2}
\varepsilon (\alpha_{\rm opt}):= \frac{T}{2\sigma_g^2}+\frac{2\sqrt{T\log (1/\delta)}}{\sqrt{2}\cdot\sigma_g}.
\tag{A.2}
\end{equation}
Finally, letting the two parts of RHS of (\ref{A.2}) both smaller than $\varepsilon/2$, we obtain the following choice of $\sigma_g$
$$\sigma_g\geq \max\{\sqrt{\frac{T}{\varepsilon}},\; \frac{2\sqrt{2T\log (1/\delta)}}{\varepsilon}\}.$$

As a result, Algorithm \ref{fedpowerfull} attains $(\varepsilon, \delta)$-DP with the variance of Gaussian noise $\nu:=\frac{2\sqrt{r}m}{n}\sigma_g\geq  \frac{2\sqrt{r}m}{n}\max\{\sqrt{\frac{T}{\varepsilon}},\; \frac{2\sqrt{2T\log (1/\delta)}}{\varepsilon}\}$.
\QEDA

\subsubsection*{Proof of Theorem \ref{utilitythm}}
\emph{Proof sketch of Theorem \ref{utilitythm}}:  First, we define a virtual sequence
\[\overline{Z}_t = \frac{1}{m}\sum_{i=1}^m  Z_t^{(i)} O_t^{(i)}.\]
Here $O_t^{(i)} \in \mathbb R^{r \times r}$ is defined as
\begin{equation}
O_t^{(i)} =
\begin{cases}
\mathbb I_r & \text{if} \ t \in \mathcal I_T^p\\
D_{t+1}^{(i)} & \text{if} \ t \notin \mathcal I_T^p.
\end{cases}
\nonumber
\end{equation}
Then, we will write $\overline{Z}_t$ in the following recursive manner,
\[\overline{Z}_t=[A\overline{Z}_{t-1}+\mathcal G_t]R_t^{-1},\]
where $R_t$ is a reversible matrix to be defined, and $\mathcal G_t$ is some noisy perturbation coming from the Gaussian noise incurred by DP and local iterates. To analyze the convergence of \textsf{FedPower}, we use the analytical framework of noisy power iterates in \citep{hardt2014noisy}. However, their results require $\overline{Z}_t$ to have orthonormal columns, which is not met in our setting. As a remedy, we obtain the following results, which is a modification of Corollary 1.1 (see Lemma \ref{lemma1}) in \citep{hardt2014noisy}.

\begin{lem}[Informal version of Lemma \ref{lem:errcor}]
Let $\overline{Z}_0$ be the orthonormalized space of a $d\times r$ random matrix with each entry being i.i.d. standard Gaussian. Assume $\overline{Z}_t$ iterates as follows,
\begin{equation}
\phantomsection
\overline{Z}_t\leftarrow  A \overline{Z}_{t-1}+\mathcal G_t.\nonumber
\end{equation}
If $\mathcal G_t$ satisfies
\begin{equation}
\phantomsection
5\| \mathcal G_t\|_2\leq \epsilon (\sigma_k-\sigma_{k+1}){\rm min}_t\|\overline{Z}_t\|_{\tiny
{\rm m}}\quad{\rm and}\quad 5\|U_k^\intercal \mathcal G_t\|_2\leq (\sigma_k-\sigma_{k+1}){\rm min}_t\|\overline{Z}_t\|_{\tiny
{\rm m}}\frac{\sqrt{r}-\sqrt{k-1}}{\tau \sqrt{d}},\nonumber
\end{equation}
for some fixed $\tau$ and $\epsilon <1/2$. Then with high probability, there exists an $T=O(\frac{\sigma_k}{\sigma_k-\sigma_{k+1}}{\rm log}(d\tau/\epsilon))$ so that after $T$ steps $$\|(\mathbb I-\overline{Z}_T\overline{Z}_T^\intercal)U_k\|_2\leq \epsilon.$$
The result also holds for the following iterates with any reversible matrix $R_t$, \[\overline{Z}_t\leftarrow [A \overline{Z}_{t-1}+\mathcal G_t]R_t^{-1}.\]
\end{lem}
In light of this result, the convergence of Algorithm \ref{fedpowerfull} could be established if we could bound the perturbation error induced from un-synchronization and DP. \QEDA

\begin{proof}
We provide a proof in three steps.
	
\paragraph{First step: Perturbed iterate analysis.}
Recall that we defined a virtual sequence by \[\overline{Z}_t = \frac{1}{m}\sum_{i=1}^m Z_t^{(i)} O_t^{(i)},\]
where $O_t^{(i)}$ is $\mathbb I_r$ if $t\in\mathcal I_T^p$, and is $D_{t+1}^{(i)}$ defined by
\[D_{t+1}^{(i)}=\underset{D\in \mathcal F\cap\mathcal O_r }{{\rm argmin}}\;\|Z_{t}^{(i)}D-Z_{t}^{(1)}\|_o,\]
if $t\notin\mathcal I_T^p$. Now we discuss the recursive iteration of $\overline{Z}_t$ under $t\notin \mathcal I_T^p$ and $t\in \mathcal I_T^p$, respectively. We will use repeatedly the fact that $Y_t^{(i)}=Z_t^{(i)}R_t^{(i)}$ in the following proof.

When $t\notin \mathcal I_T^p$, we note that $Y_t^{(i)}=A_iZ_{t-1}^{(i)}+N_{t}^{(i)}$,
where $N_{t}^{(i)}$ has i.i.d. $\mathcal N(0,\nu^2)$ entries.
Then, given any invertible $R_t$ (to be specified in Lemma \ref{lemmart}), we have
\begin{align}
\phantomsection
\label{A.3}
\overline{Z}_{t}&=\frac{1}{m}\sum_{i=1}^m  Z^{(i)}_{t}O^{(i)}_t\nonumber\\
&=\frac{1}{m}\sum_{i=1}^m A_i Z^{(i)}_{t-1}O^{(i)}_{t-1}R_t^{-1}+\frac{1}{m}\sum_{i=1}^m Z^{(i)}_{t}[O_t^{(i)}R_t-R_t^{(i)}O_{t-1}^{(i)}]R_t^{-1}+\frac{1}{m}\sum_{i=1}^m  N^{(i)}_{t}O_{t-1}^{(i)} R^{-1}_t \nonumber\\
&:=(A\overline{Z}_{t-1}+H_t+W_t+N_t)R_t^{-1},
\tag{A.3}
\end{align}
where
\begin{align}
\phantomsection
\label{A.4}
H_t:=\frac{1}{m}\sum_{i=1}^m H_t^{(i)}:=\frac{1}{m}\sum_{i=1}^m (A_i-A)Z_{t-1}^{(i)}O_{t-1}^{(i)},\nonumber
\tag{A.4}
\end{align}
\begin{align}
\phantomsection
\label{A.5}
W_t:=\frac{1}{m}\sum_{i=1}^m W_t^{(i)}:=\frac{1}{m}\sum_{i=1}^m Z_{t}^{(i)}[O_{t}^{(i)}R_t-R^{(i)}_tO_{t-1}^{(i)}],\nonumber
\tag{A.5}
\end{align}
and
\begin{align*}
\phantomsection
N_t=\frac{1}{m}\sum_{i=1}^m N_t^{(i)}O_{t-1}^{(i)}.
\end{align*}

When $t\in \mathcal I_T^p$, synchronization happens. In such case, $R^{(i)}_t$'s are identical for all $i\in [m]$ and we let them be $R_t$; see Lemma \ref{lemmart} for details, and
\[Y^{(i)}_t=\frac{1}{m}\sum_{l=1}^m [A_l Z^{(l)}_{t-1}D^{(l)}_{t}+ N_t^{(l)}D_t^{(l)}],\]
for all $i\in [m]$. Hence,
\begin{align}
\phantomsection
\label{A.6}
\overline{Z}_{t}&=\frac{1}{m}\sum_{i=1}^m   Z^{(i)}_{t}O^{(i)}_t\nonumber\\
&=(\frac{1}{m}\sum_{i=1}^m A_i Z^{(i)}_{t-1}D^{(i)}_{t}+\frac{1}{m}\sum_{i=1}^m  Z^{(i)}_{t}[O_t^{(i)}R_t-R_t^{(i)}]+\frac{1}{m}\sum_{i=1}^m N_t^{(i)}D_t^{(i)})R_t^{-1}\nonumber\\
&=(\frac{1}{m}\sum_{i=1}^m A_i Z^{(i)}_{t-1}O^{(i)}_{t-1}+\frac{1}{m}\sum_{i=1}^m Z^{(i)}_{t}[R_t-R_t^{(i)}]+\frac{1}{m}\sum_{i=1}^m N_t^{(i)}D_t^{(i)})R_t^{-1}\nonumber\\
&:=(A\overline{Z}_{t-1}+H_t+N'_t)R_t^{-1},
\tag{A.6}
\end{align}
where we used the fact that $O_t^{(i)}=\mathbb I_r$, $O_{t-1}^{(i)}=D^{(i)}_{t}$, and $R_t=R_t^{(i)}$; $H_t$ is defined in (\ref{A.4}), and
$$N_t':=\frac{1}{m}\sum_{i=1}^m N_t^{(i)}D_t^{(i)}.$$

\paragraph{Second step: Bound the noise term.}
We proceed to bound $\|H_t\|_2,\|W_t\|_2$, $\|N_t\|_2$ and $\|N'_t\|_2$, respectively.

$\bullet$ For $\|H_t\|_2$, we have
\begin{align}
\phantomsection
\label{A.7}
		\| H_t \|_2
		&= \| \frac{1}{m}\sum_{i=1}^m  H_t^{(i)} \|_2
		= \frac{1}{m}\sum_{i=1}^m \| \left( A_i -A \right)  Z_{t-1}^{(i)}O_{t-1}^{(i)} \|_2\nonumber\\
		&\leq \frac{1}{m}\sum_{i=1}^m \|A_i -A\|_2\|Z_{t-1}^{(i)}O_{t-1}^{(i)}\|_2
		\leq \eta\sigma_1,
\tag{A.7}
\end{align}
where the last inequality follows from Definition \ref{assu1}.

$\bullet$ For $\| W_t \|_2$, we have
		\begin{align}
		\| W_t \|_2
		&= \big\| \frac{1}{m}\sum_{i=1}^m  W_t^{(i)} \|_2
		= \frac{1}{m}\sum_{i=1}^m \big\|  Z_t^{(i)}   \left(  O_t^{(i)}R_t - R_{t}^{(i)}O_{t-1}^{(i)} \right) \big\|_2\nonumber\\
		&\leq \frac{1}{m}\sum_{i=1}^m  \big\|O_t^{(i)}R_t - R_{t}^{(i)}O_{t-1}^{(i)} \big\|_2.\nonumber
		\end{align}
Here, a good choice of $R_t$ should be specified to ensure a tight bound of $\|W_t\|_2$. Specifically, $R_t$ is chosen in a recursive manner as we show in Lemma \ref{lemmart} in the Appendix C. In particular, we prove in Lemma \ref{lemmart} that for any $i\in[m]$
\begin{align}
\phantomsection
\label{A.8}
	\|O_{t}^{(i)} R_t - R_{t}^{(i)}O_{t-1}^{(i)}\|_2
	&\leq \sigma_1(A_1)\| Z_t^{(i)}O_{t}^{(i)}-  Z_t^{(1)} \|_2+\|A_1 - A_i\|_2  + \sigma_1(A_i)\|Z_{t-1}^{(i)}O_{t-1}^{(i)}-Z_{t-1}^{(1)}\|_2\nonumber\\
    &\leq{2} \sigma_1\left(\rho_t+\eta+(1+\eta)\rho_{t-1}\right),
\tag{A.8}
\end{align}
where we make use of Definition \ref{assu1} and define
\begin{equation}
\phantomsection
\label{A.9}
	\rho_{t}={\rm max}_i\|Z_t^{(i)}O_t^{(i)}-Z_t^{(1)}\|_2={\rm max}_i\|Z_t^{(i)}D_{t+1}^{(i)}-Z_t^{(1)}\|_2,\tag{A.9}
\end{equation}
with $\rho_t$ upper bounded as in (\ref{3.6}) and (\ref{3.7}); see Lemma \ref{lem:rho1} and \ref{lem:rho2} for details.
As a result, we obtain
\begin{equation}
\phantomsection
\label{A.10}
	\|W_t\|_2\leq 2\sigma_1\left(\rho_t+\eta+(1+\eta)\rho_{t-1}\right).\tag{A.10}
\end{equation}

$\bullet$ For $\|N_t\|_2$, we recall that
$$N_t=\frac{1}{m}\sum_{i=1}^m N_t^{(i)}O_{t-1}^{(i)},$$
where $N_t^{(i)}$ is a random matrix with each entry being i.i.d. $\mathcal N(0,\nu^2)$.
We will make use of Bernstein-type concentration inequality developed for sums of a random matrix multiplied by a non-random matrix (\citep{lei2022bias}, see also Lemma \ref{lei-bound1}) to proceed. We first observe that $Y\sim \mathcal N(0,\nu^2)$ is $(2\nu^2,\nu^2)$-Bernstein (see Definition \ref{bernstein}) because for non-negative integer $q$,
\begin{equation}
\label{A.11}
\mathbb E |Y|^q\leq \nu^q (q-1)!\leq \frac{2\nu^2}{2}q! \nu^{q-2},
\tag{A.11}
\end{equation}
where the first inequality follows from Lemma \ref{momentofgaussian}. Then by conditioning on $O_{t-1}^{(i)}$'s and Lemma \ref{lei-bound1}, we have
\begin{equation}
\label{A.12}
\mathbb P(\|N_t\|_2\geq t)=\mathbb P(\|\sum_{i=1}^m N_t^{(i)}O_{t-1}^{(i)}\|_2\geq mt)\leq 4d \exp \left(-\frac{m^2t^2 /2}{2\nu^2 d m+\nu tm}\right)
\tag{A.12}
\end{equation}
Let $$t_{\rm opt}\asymp \sqrt{4 d\nu^2 \log(dT)/m} +\nu \log(dT)/m,$$ we then have with probability larger than $1-(dT)^{-\gamma}$ for some $\gamma>1$ that
\begin{equation}
\label{A.13}
\|N_t\|_2\leq t_{\rm opt} \lesssim {\nu }\max\left\{\sqrt{\frac{d \log(dT)}{m}},\; \frac{\log(dT)}{m}\right\}.
\tag{A.13}
\end{equation}
Further, by the union bound, $\max_t\|N_t\|_2$ is also bounded by the RHS of \eqref{A.13} but with probability $1-(dT)^{-\alpha}$ for some $\alpha >0$.

$\bullet$ For $\|N_t'\|_2$, we can use the same technique as that for bounding $\|N_t\|_2$ and the same bound can be derived.

Combining (\ref{A.7}), (\ref{A.10}) and (\ref{A.13}) and recalling the expression (\ref{A.3}) and (\ref{A.6}), we obtain that the perturbation noise $\mathcal G_t:=H_t+W_t+N_t+N'_t$ satisfies that
\begin{align}
\phantomsection
\label{A.14}
{\rm max}_t\|\mathcal G_t\|_2\lesssim {\nu }\max\left\{\sqrt{{d \log(dT)}/{m}},\; {\log(dT)}/{m}\right\}+ 2\sigma_1\left(\eta+(2+\eta){\rm max}_t\rho_{t}\right),\nonumber
\tag{A.14}
\end{align}
with probability larger than $1-(dT)^{-\alpha}$ for some $\alpha >0$. To lighten the notation, we denote the RHS of (\ref{A.14}) by ${\rm Err} (\nu, \sigma_1,d, m, T, p,k,\eta)$.

\paragraph{Third step: Establish convergence.}
Now we make use of the result in Lemma \ref{lem:errcor} to establish convergence. Note that in Lemma \ref{lem:errcor}, there still exists an unknown term $\|\overline{Z}_t\|_{\rm \tiny m}$. We prove in Lemma \ref{lem:z} that
\[\|\overline{Z}_t\|_{\rm \tiny m}\geq 1-(1-1/m){\rm max}_t\rho_t.\]
Denote $$\epsilon':= \frac{{\rm Err} (\nu, \sigma_1,d, m, T, p,k,\eta)}{(\sigma_k-\sigma_{k+1})(1-(1-1/m){\rm max}_t\rho_t)},$$ then by (\ref{A.14}),
\begin{align}
\phantomsection
\label{A.15}
5{\rm max}_t\|\mathcal G_t\|_2\leq \epsilon'(\sigma_k-\sigma_{k+1})\|\overline{Z}_t\|_{\rm \tiny m}.\nonumber
\tag{A.15}
\end{align}
Hence the first condition in Lemma \ref{lem:errcor} is satisfied. For the second condition, we have that
$$5{\rm max}_t\|U_k^\intercal \mathcal G_t\|_2\leq 5{\rm max}_t\|\mathcal G_t\|_2,$$
which implies that the second condition would be met automatically if $ \epsilon'<\frac{\sqrt{r}-\sqrt{k-1}}{\tau \sqrt{d}}$, which is our condition. Consequently, by Lemma \ref{lem:errcor}, we have after $T=O(\frac{\sigma_k}{\sigma_k-\sigma_{k+1}}{\rm log}(\frac{d}{\epsilon'}))$ iterations,
$$\|(\mathbb I_d-\overline{Z}_T\overline{Z}_T^\intercal)U_k\|_2\leq \epsilon',$$
with probability larger than $1-(dT)^{-\alpha}-\tau^{-\Omega (r+1-k)}-e^{-\Omega(d)}$.
\end{proof}

\subsubsection*{Proof of Theorem \ref{thm:rho}}
By Lemma \ref{lem:rho1} and \ref{lem:rho2}, the following results hold.
\begin{itemize}
		\item If $\mathcal F = \mathcal O_{r}$, then
		\[
		\rho_t
		\le \sqrt{2} \min\left\{\frac{2\kappa^ p \eta (1+\eta)^{p-1} }{(1-\eta)^p} , \frac{\eta\sigma_1}{\delta_k} + 2 \gamma_k^{p/4}  \max_{i \in [m]}\tan \theta_k(Z_{\tau(t)}, U_{k}^{(i)})  \right\}.
		\]
		with the parameters $\delta_k={\rm min}_{i\in[m]}\delta_k^{(i)}$ with $\delta_k^{(i)} = {\rm min} \{ |\sigma_j(A_i) - \sigma_{k}(A)| :j \geq k+1 \}$ and $\gamma_k={\rm max}\{{\rm max}_{i\in[m]} \frac{\sigma_{k+1}(A_i)}{\sigma_{k}(A_i)},\frac{\sigma_{k+1}(A)}{\sigma_{k}(A)}\}$. By requiring $\eta\kappa \leq 1/3$, and Wely's inequality, we have
 \[\sigma_{{j}}(A) - \sigma_{k}(A)-\frac{1}{3}\sigma_{d}(A)\leq\sigma_j(A_i) - \sigma_{k}(A)\leq \sigma_{{j}}(A) - \sigma_{k}(A)+\frac{1}{3}\sigma_{d}(A).\]
 Hence, $\delta_k\asymp  \sigma_{k}(A)-\sigma_{k+1}(A) $.
 By requiring $\eta \leq 1/p$, we have
		$\frac{ (1+\eta)^{p-1}}{(1-\eta)^p} \leq \frac{ (1+1/p)^{p-1}}{(1-1/p)^p} \leq \mathrm{e}^2$.
		Define $C_t = \max_{i \in [m]}{\rm tan} \theta_k(Z_{\tau(t)}, U_{k}^{(i)})$.
		Actually, we have shown in Theorem \ref{utilitythm} that $\sin\theta_k(Z_{\tau(t)}, U_{k})\leq \epsilon'$ for sufficiently large $t$. By the condition that $\epsilon'$ is small enough, we can obtain $\lim\limits_{t \to \infty}\theta_k(Z_{\tau(t)}, U_{k}) = 0$.
		Then, we have
		{\small\[
		\limsup\limits_{t \to \infty} C_t =
		\limsup\limits_{t \to \infty} \max_{i \in [m]} {\rm tan}\,\theta_k(Z_{\tau(t)}, U_{k}^{(i)} )
		\leq\max_{i \in [m]}  {\tan}\,\theta_k(U_{k}, U_{k}^{(i)})
		\leq\max_{i \in [m]}  {\tan} \arg\sin \frac{\eta\sigma_1}{\delta_k}
		= O(\eta),
		\]}
where the last inequality follows from the Davis-Kahan theorem (see Lemma \ref{lem:DK}).
		\item If $\mathcal F = \{\mathbb I_r\}$, then
		\[
		\rho_t \leq 4\sqrt{2k}p \kappa^p \eta (1+\eta)^{p-1}
		\le 4\mathrm{e}\sqrt{2k}p \kappa^p \eta,
		\]
where the last inequality requires $\eta \leq 1/p$.
	\end{itemize}
Simply put together, we confirm that the bounds of $\rho_t$ in Theorem \ref{thm:rho} hold.
\QEDA

\subsection*{B. Partial participation}

\subsubsection*{Proof of Theorem \ref{utilitythmpartial}}
\emph{Proof sketch of Theorem \ref{utilitythmpartial}}:
Similar to the proof under the full participation setting, we consider the virtual sequence
\begin{equation}
\label{B.1}
\overline{Z}_{t}:= \frac{1}{K}\sum_{i\in \mathcal S_{\tau(t)}}Z_t^{(i)} O_t^{(i)},
\tag{B.1}
\end{equation}
where $O_t^{(i)}$ is $\mathbb I_r$ if $t\in\mathcal I_T^p$, and is $D_{t+1}^{(i)}$ defined by
\[D_{t+1}^{(i)}=\underset{D\in \mathcal F\cap\mathcal O_r }{{\rm argmin}}\;\|Z_{t}^{(i)}D-Z_{t}^{(1)}\|_o,\]
if $t\notin\mathcal I_T^p$.
We will write $\overline{Z}_t$ in the following recursive manner,
\[\overline{Z}_t=[A\overline{Z}_{t-1}+\mathcal G_t]R_t^{-1},\]
where $R_t^{-1}$ is the a reversible matrix to be specified, and $\mathcal G_t$ is some noisy perturbations which turns out to come from three sources. Except for the DP's noise and the local iterates' perturbation, the random sampling of local machines also contributes to $\mathcal G_t$. To be specific, when $\tau(t)\neq \tau(t-1)$, it holds with high probability that $\mathcal S_{\tau(t)}\neq \mathcal S_{\tau(t-1)}$. Thus we also need to bound the bias term that the random sampling of local machines brings. With $\mathcal G_t$ properly bounded, the convergence of would be established using Lemma \ref{lem:errcor}. \QEDA

\begin{proof}
We provide a proof in three steps.
	
\paragraph{First step: Perturbed iterate analysis.}
Now we proceed to derive the recursive iteration of $\overline{Z}_t$ under $t\notin \mathcal I_T^p$ and $t\in \mathcal I_T^p$, respectively. We will use repeatedly the fact that $Y_t^{(i)}=Z_t^{(i)}R_t^{(i)}$ in the following proof.

When $t\notin \mathcal I_T^p$, we have $Y_t^{(i)}=A_iZ_{t-1}^{(i)}+N_{t}^{(i)}$, where $N_{t}^{(i)}$ has i.i.d. $\mathcal N(0,\nu^2)$ entries.
Thus, given any invertible $R^t$ (to be specified in Lemma \ref{lemmart}), we have
\begin{align}
\phantomsection
\label{B.2}
\overline{Z}_{t}&=\frac{1}{K}\sum_{i\in \mathcal S_{\tau(t)}}Z_t^{(i)} O_t^{(i)}\nonumber\\
&=\frac{1}{K}\sum_{i\in \mathcal S_{\tau(t)}}A Z^{(i)}_{t-1} O_{t-1}^{(i)}R_t^{-1}+\frac{1}{K}\sum_{i\in \mathcal S_{\tau(t)}}(A_i-A) Z^{(i)}_{t-1}O_{t-1}^{(i)}R_t^{-1}\nonumber\\
&\quad\quad\quad\quad\quad+\frac{1}{K}\sum_{i\in \mathcal S_{\tau(t)}} Z^{(i)}_{t}[O_{t}^{(i)}R_t-R_t^{(i)}O_{t-1}^{(i)}]R^{-1}_t+\frac{1}{K}\sum_{i\in \mathcal S_{\tau(t)}}N_t^{(i)} O_{t-1}^{(i)}R^{-1}_t\nonumber\\
&:=(\mathcal J_t+ H_t +  W_t + N_t)R_t^{-1},
\tag{B.2}
\end{align}
for which $\mathcal J_t$ could be further expressed as
{\scriptsize
\begin{align}
\phantomsection
\label{B.3}
\mathcal J_t&=\frac{1}{K}\sum_{i\in \mathcal S_{\tau(t-1)}}A Z^{(i)}_{t-1}O_{t-1}^{(i)}+\frac{1}{K}\left(\sum_{i\in \mathcal S_{\tau(t)}}A Z^{(i)}_{t-1}O_{t-1}^{(i)}-\sum_{i\in \mathcal S_{\tau(t-1)}}A Z^{(i)}_{t-1}O_{t-1}^{(i)}\right)\nonumber\\
&=A\overline{Z}_{t-1}+ \left(\sum_{i\in \mathcal S_{\tau(t)}}A Z^{(i)}_{t-1}O_{t-1}^{(i)}/K-\mathbb E_{\mathcal S_{\tau(t)}}[\sum_{i\in \mathcal S_{\tau(t)}}A Z^{(i)}_{t-1}O_{t-1}^{(i)}/K]\right)\nonumber\\
&\quad\quad\quad\quad\quad\quad\quad\quad\quad\quad\quad\quad\quad\quad\quad\quad\quad\quad\quad+\left(\mathbb E_{\mathcal S_{\tau(t)}}[\sum_{i\in \mathcal S_{\tau(t)}}A Z^{(i)}_{t-1}O_{t-1}^{(i)}/K]-\sum_{i\in \mathcal S_{\tau(t-1)}}A Z^{(i)}_{t-1}O_{t-1}^{(i)}/K\right)\nonumber\\
&=A\overline{Z}_{t-1}+ \left(\sum_{i\in \mathcal S_{\tau(t)}}A Z^{(i)}_{t-1}O_{t-1}^{(i)}/K-\sum_{i=1}^m A Z_{t-1}^{(i)}O_{t-1}^{(i)}/m\right)+\left(\sum_{i=1}^m A Z_{t-1}^{(i)}O_{t-1}^{(i)}/m-\sum_{i\in \mathcal S_{\tau(t-1)}}A Z^{(i)}_{t-1}O_{t-1}^{(i)}/K\right)\nonumber\\
&:=A\overline{Z}_{t-1}+ E_t+ F_t,
\tag{B.3}
\end{align}
}

\noindent where the expectation only involves the randomness from sampling the local machines.
Combining (\ref{B.3}) with (\ref{B.2}), we obtain when $t\notin \mathcal I^p_T$ that,
\begin{align}
\phantomsection
\label{B.4}
\overline{Z}_t=(A\overline{Z}_{t-1}+E_t+F_t+H_t+W_t+N_t)R_t^{-1}.
\tag{B.4}
\end{align}

When $t\in \mathcal I^p_T$, the synchronization happens. Thereby, for all $i\in \mathcal S_{\tau(t)}$,
\begin{align}
\phantomsection
\label{B.5}
Y_t^{(i)}&=\frac{1}{K}\sum_{l\in \mathcal S_{\tau(t)}}(A_l Z_{t-1}^{(l)}D_{t}^{(l)} +N_{t}^{(l)}D_t^{(l)}):=\frac{1}{K}\sum_{l\in \mathcal S_{\tau(t)}}A_l Z_{t-1}^{(l)}D_{t}^{(l)}+N'_t.\nonumber
\tag{B.5}
\end{align}
Using similar treatments as in (\ref{B.2}) and (\ref{B.3}), we have when $t\in \mathcal I^p_T$ that,
\begin{align}
\phantomsection
\label{B.6}
\overline{Z}_t=(A\overline{Z}_{t-1}+E_t+F_t+H_t+N'_t)R_t^{-1},
\tag{B.6}
\end{align}
where note that similar to the full participation protocol, $W_t$ does not appear when $t\in \mathcal I^p_T$.

\paragraph{Second step: Bound the noise term.}
We proceed to bound $E_t$, $F_t$, $H_t$, $W_t$, $N_t$ and $N'_t$, respectively.

$\bullet$ For $\|E_t\|_2$, we denote
\begin{align}
\phantomsection
\label{B.7}
E_t&:=A\sum_{i\in \mathcal S_{\tau(t)}}\frac{1}{K}\left (Z^{(i)}_{t-1}O^{(i)}_{t-1}-\frac{1}{m}\sum_{i=1}^m Z_{t-1}^{(i)}O^{(i)}_{t-1}\right):=A \sum_{i\in \mathcal S_{\tau(t)}} S_t^{(i)}:= A S_t.\nonumber
\tag{B.7}
\end{align}
We will make use of the matrix Bernstein inequality (\citep{tropp2015introduction}, restated in Lemma \ref{lemmabernstein}), to bound $\|S_t\|_2$. Consider only the randomness that the $S_{\tau(t)}$ brings, we have $$\mathbb E_{S_{\tau(t)}}(S_t^{(i)})=0\quad{\rm and}\quad \|S_t^{(i)}\|_2\leq \frac{2}{K},$$
and let
\begin{align}
\phantomsection
\label{B.8}
\nu(S_t)={\rm max}\{\|\mathbb E(S_tS_t^\intercal)\|_2, \|\mathbb E(S_t^\intercal S_t)\|_2\}.
\tag{B.8}
\end{align}
By Lemma \ref{lemmabernstein}, we thus have
\begin{align}
\phantomsection
\label{B.9}
\mathbb P(\|S_t\|_2\geq t)\leq(d+r)\cdot  \exp \left(\frac{-t^2/2}{\nu(S_t)+ {2t}/{(3K)}}\right).
\tag{B.9}
\end{align}
Choosing {$t=\sqrt{2\nu(S_t)\log(T(d+r))}+{2\log(T(d+r))}/{(3K)}$}, then (\ref{B.9}) implies
\begin{align}
\phantomsection
\label{B.10}
\|S_t\|_2 \lesssim \sqrt{2\nu(S_t)\log(T(d+r))}+{2\log(T(d+r))}/{(3K)} ,
\tag{B.10}
\end{align}
with probability larger than $1-(T(d+r))^{-\mu}$ for some $\mu>1$. It remains to bound $\nu(S_t)$.
Denote $\xi=\frac{1}{m}\sum_{i=1}^m  Z_{t-1}^{(i)}O^{(i)}_{t-1}$, then
{\small
\begin{align}
\phantomsection
\label{B.11}
S_tS_t^\intercal = \frac{1}{K^2}\sum_{i\in \mathcal S_{\tau(t)}}Z^{(i)}_{t-1}O^{(i)}_{t-1}  \sum_{i\in \mathcal S_{\tau(t)}}(Z^{(i)}_{t-1}O^{(i)}_{t-1})^\intercal -\xi \frac{1}{K}\sum_{i\in \mathcal S_{\tau(t)}}(Z^{(i)}_{t-1}O^{(i)}_{t-1})^\intercal-\frac{1}{K}\sum_{i\in \mathcal S_{\tau(t)}}Z^{(i)}_{t-1}O^{(i)}_{t-1} \xi^\intercal+ \xi\xi^\intercal.
\tag{B.11}
\end{align}
}

\noindent To shorten the notation, we denote $Z^{o(i)}_{t-1}:=Z^{(i)}_{t-1}O^{(i)}_{t-1}$.
Taking expectation with respect to $\mathcal S_{\tau(t)}$, we then have
{\small
\begin{align}
\phantomsection
\label{B.12}
\mathbb E_{\mathcal S_{\tau(t)}}S_tS_t^\intercal &= \mathbb E_{\mathcal S_{\tau(t)}}[ \frac{1}{K^2}\sum_{i\in \mathcal S_{\tau(t)}}Z^{o(i)}_{t-1} \sum_{i\in \mathcal S_{\tau(t)}}(Z^{o(i)}_{t-1})^\intercal] - \xi\xi^\intercal\nonumber\\
&=\frac{1}{K^2} \mathbb E_{\mathcal S_{\tau(t)}}[\sum_{i\neq j, i,j\in \mathcal S_{\tau(t)} }Z^{o(i)}_{t-1}(Z^{o(j)}_{t-1})^\intercal +\sum_{i=j, i,j\in \mathcal  S_{\tau(t)} }Z^{o(i)}_{t-1}(Z^{o(j)}_{t-1})^\intercal]- \xi\xi^\intercal\nonumber\\
&=\frac{K(K-1)}{K^2}\sum_{i\neq j, i,j=1}^m \frac{1}{m^2}Z^{o(i)}_{t-1}(Z^{o(j)}_{t-1})^\intercal +\frac{K}{K^2}\sum_{i=j=1}^m \frac{1}{m} Z^{o(i)}_{t-1}(Z^{o(j)}_{t-1})^\intercal-\sum_{i,j=1}^m \frac{1}{m^2} Z^{o(i)}_{t-1}(Z^{o(j)}_{t-1})^\intercal\nonumber\\
&=-\frac{1}{K}\sum_{i\neq j, i,j=1}^m \frac{1}{m^2}Z^{o(i)}_{t-1}(Z^{o(j)}_{t-1})^\intercal +\sum_{i=1}^m \frac{1}{m}(\frac{1}{K}-\frac{1}{m})Z^{o(i)}_{t-1}(Z^{o(i)}_{t-1})^\intercal.
\tag{B.12}
\end{align}
}
As a result,
\begin{align}
\phantomsection
\label{B.13}
\|\mathbb E_{\mathcal S_{\tau(t)}}(S_tS_t^\intercal) \|_2\leq \frac{1}{K}\sum_{i\neq j, i,j=1}^m \frac{1}{m^2} +\sum_{i=1}^m \frac{1}{m^2}+ \frac{1}{K}\lesssim \frac{1}{K}.
\tag{B.13}
\end{align}
Similarly, we could obtain $$\|\mathbb E_{\mathcal S_{\tau(t)}}(S_t^\intercal S_t) \|_2\lesssim \frac{1}{K}.$$
Consequently, by (\ref{B.10}), (\ref{B.7}), and the union bound, we have
\begin{align}
\phantomsection
\label{B.14}
\underset{t\in \mathcal I_T^p}{\rm max}\|E_t+F_t\|_2\leq \underset{t}{\rm max}2\|E_t\|_2\lesssim \sigma_1\max\left\{\sqrt{\frac{\log(dT)}{K}},\; \frac{\log(dT)}{K}\right\}.
\tag{B.14}
\end{align}
with probability larger than $1-(T(d+r))^{-\iota}$ for some $\iota>0$.

$\bullet$ For $\|H_t\|_2$, similar to the deviation of (\ref{A.7}), we have
\begin{equation}
\label{B.15}
\|H_t\|_2=\|\frac{1}{K}\sum_{i\in \mathcal S_{\tau(t)}}(A_i-A) Z^{(i)}_{t-1}O^{(i)}_{t-1}\|_2\leq \eta \sigma_1.
\tag{B.15}
\end{equation}

$\bullet$ For $\|W_t\|_2$, similar to the deviation of (\ref{A.10}), we have
\begin{equation}
\label{B.16}
\|W_t\|_2=\|\frac{1}{K}\sum_{i\in \mathcal S_{\tau(t)}} Z^{(i)}_{t-1}(O_t^{(i)}R_t-R_t^{(i)}O_{t-1}^{(i)})\|_2\leq 2\sigma_1\left(\eta+(2+\eta){\rm max}_t\rho_{t}\right),
\tag{B.16}
\end{equation}
where $\rho_t$ is defined as
\begin{equation}
\phantomsection
	\rho_{t}={\rm max}_i\|Z_t^{(i)}O_t^{(i)}-Z_t^{(1)}\|_2={\rm max}_i\|Z_t^{(i)}D_{t+1}^{(i)}-Z_t^{(1)}\|_2,\nonumber
\end{equation}
and upper bounded as in (\ref{3.6}) and (\ref{3.7}); see Lemma \ref{lem:rho1} and \ref{lem:rho2} for details.

$\bullet$ For $\|N_t\|_2$, conditioning on $\mathcal S_{\tau(t)}$, using similar tools as for (\ref{A.13}), we have
\begin{align}
\phantomsection
\label{B.17}
{\rm max}_t \|N_t\|_2 \lesssim {\nu }\max\left\{\sqrt{\frac{d \log(dT)}{K}},\; \frac{\log(dT)}{K}\right\}.\nonumber
\tag{B.17}
\end{align}
with probability larger than $1-(dT)^{-\gamma}$ for some $\gamma >0$.

Putting (\ref{B.14})-(\ref{B.17}) together, and recalling (\ref{B.4}) and (\ref{B.6}), then the perturbation noise $\mathcal G_t:= E_t+F_t+H_t+W_t+N_t+N'_t$ satisfies that
\begin{align}
\phantomsection
\label{B.19}
{\rm max}_t &\|\mathcal G_t\|_2\lesssim {\nu }\max\left\{\sqrt{\frac{d \log(dT)}{K}},\; \frac{\log(dT)}{K}\right\}+2\sigma_1\left(\eta+(2+\eta){\rm max}_t \rho_{t}\right)\nonumber\\
&\quad\quad\quad\quad\quad\quad\quad\quad\quad+\sigma_1\max\left\{\sqrt{\frac{\log(dT)}{K}},\; \frac{\log(dT)}{K}\right\},\nonumber
\tag{B.19}
\end{align}
with probability larger than $1-(dT)^{-\beta}$ for positive constant $\beta$. For notational simplicity, we denote the RHS of (\ref{B.19}) as ${\rm Err} (\nu, \sigma_1,d, m, K, T, p,k,\eta)$

\paragraph{Third step: Establish convergence.}
At last, we use Lemma \ref{lem:errcor} to establish convergence.
Denote $$\epsilon'':= \frac{{\rm Err} (\nu, \sigma_1,d, m, K, T, p,k,\eta)}{(\sigma_k-\sigma_{k+1})(1-{\rm max}_t\rho_t)},$$ then by (\ref{B.19}),
\begin{align}
\phantomsection
5{\rm max}_t\|\mathcal G_t\|_2\leq \epsilon''(\sigma_k-\sigma_{k+1})\|\overline{Z}_t\|_{\rm \tiny m},\nonumber
\end{align}
where we used the fact in Lemma \ref{lem:z} that
\[\|\overline{Z}_t\|_{\rm \tiny m}\geq 1-{\rm max}_t\rho_t.\]
The first condition in Lemma \ref{lem:errcor} is thus satisfied. For the second condition, we have that
$$5{\rm max}_t\|U_k^\intercal \mathcal G_t\|_2\leq 5{\rm max}_t\|\mathcal G_t\|_2,$$
which implies that the second condition would be met automatically if $ \epsilon''<\frac{\sqrt{r}-\sqrt{k-1}}{\tau \sqrt{d}}$, which is our condition. Overall, by Lemma \ref{lem:errcor}, we have after $T=O(\frac{\sigma_k}{\sigma_k-\sigma_{k+1}}{\rm log}(\frac{d}{\epsilon''}))$ iterations,
$$\|(\mathbb I_d-\overline{Z}_T\overline{Z}_T^\intercal)U_k\|_2\leq \epsilon'',$$
with probability larger than $1-(dT)^{-\beta}-\tau^{-\Omega (r+1-k)}-e^{-\Omega(d)}$ for some positive constants $\beta$ and $\tau$.
\end{proof}

\subsection*{C. Technical lemmas}
The following lemma is a variant of Lemma 2.2 in \citep{hardt2014noisy} (see Lemma \ref{lemhardt}). Given the relation $\overline{Z}_t=A\overline{Z}_{t-1}+\mathcal G_t$, they require $\overline{Z}_t$ to have orthonormal columns, i.e., $\overline{Z}_t^\intercal \overline{Z}_t=\mathbb I_r$. However, it is impossible in our analysis. As a remedy, we slightly change the lemma to allow arbitrary $\overline{Z}_t$. This will also change the condition on $\mathcal G_t$.
\begin{lem}
\label{lem:err}
Let $U_k\in \mathbb R^{d\times k}$ denote the top $k$ eigenvectors of $A$ and let $\sigma_1\geq...\geq \sigma_d $ denote its eigenvalues. Let $\overline{Z}_t\in \mathbb R^{ d\times r}$ for some $r\geq k$. Let $\mathcal G_t$ satisfy
\begin{equation}
\phantomsection
4\|U_k^\intercal \mathcal G_t\|_2\leq (\sigma_k-\sigma_{k+1}){\rm cos}\,\theta_k(U_k, \overline{Z}_t)\|\overline{Z}_t\|_{\tiny
{\rm m}}\quad{\rm and}\quad 4\| \mathcal G_t\|_2\leq (\sigma_k-\sigma_{k+1})\|\overline{Z}_t\|_{\tiny
{\rm m}}\epsilon,\nonumber
\end{equation}
for some $\epsilon\leq 1$, where $\|\overline{Z}_t\|_{\tiny
{\rm m}}$ denotes the minimum singular value of $\overline{Z}_t$. Then
$${\rm tan}\,\theta_k(U_k,A \overline{Z}_t+\mathcal G_t )\leq {\rm max} \left(\epsilon,{\rm max}\left(\epsilon,\left(\frac{\sigma_{k+1}}{\sigma_k}\right)^{1/4}\right){\rm tan}\,\theta_k(U_k,\overline{Z}_t)\right),$$
where the LHS can be replaced by ${\rm tan}\,\theta_k(U_k,(A\overline{Z}_t+\mathcal G_t)R^{-1}_t)$ with any reversible matrix $R_t$.
\end{lem}
\begin{proof}
The proof actually follows closely from that of \citep{hardt2014noisy}. Hence, we here only show the main steps.

First, by the definition of angles between subspaces, the Lemma 2.2 in \citep{hardt2014noisy} obtain that,
\[{\rm tan}\,\theta_k(U_k,A \overline{Z}_t+\mathcal G_t )\leq \underset{\|w\|_2=1, \Pi ^* w=w}{{\rm max}}\frac{1}{\|U_k^\intercal \overline{Z}_tw\|_2}\cdot \frac{\sigma_{k+1}\|(U_k^\perp)^\intercal \overline{Z}_tw \|_2+\|(U_k^\perp)^\intercal \mathcal G_t w \|_2}{\sigma_{k}-\|U_k^\intercal \mathcal G_t w \|_2/\|U_k^\intercal \overline{Z}_t w \|_2},\]
where $\Pi ^*$ is the matrix projecting onto the smallest $k$ principal angles of $\overline{Z}_t$. Define $\Delta=(\sigma_{k}-\sigma_{k+1})/4$. Then, by the assumption on $\mathcal G_t$,
\[\underset{\|w\|_2=1, \Pi ^* w=w}{{\rm max}}\frac{\|U_k^\intercal \mathcal G_t w \|_2}{\|U_k^\intercal \overline{Z}_t w \|_2}\leq \frac{\|U_k^\intercal \mathcal G_t \|_2}{{\rm cos}\,\theta_k(U_k, \overline{Z}_t)\|\overline{Z}_t\|_{\tiny
{\rm m}}}\leq \Delta, \]
where we used Fact \ref{fact2} on the principle angle in Appendix E. Similarly, using the fact that ${1/{\rm cos}\,\theta} \leq 1+{\rm tan}\,\theta$ for any angle $\theta$, we have
\[\underset{\|w\|_2=1, \Pi ^* w=w}{{\rm max}}\frac{\|(U_k^\perp)^\intercal \mathcal G_t w \|_2}{\|U_k^\intercal \overline{Z}_t w \|_2}\leq \frac{\| \mathcal G_t \|_2}{{\rm cos}\,\theta_k(U_k, \overline{Z}_t)\|\overline{Z}_t\|_{\tiny
{\rm m}}}\leq \epsilon\Delta(1+{\rm tan}\,\theta_k(U_k, \overline{Z}_t)). \]
Given the above two inequalities, the remaining proofing strategy is the same as that of \citep{hardt2014noisy}. Hence we here omit it.
In addition, noting the Fact \ref{fact} in Appendix E, we know that the result can be generalized to ${\rm tan}\,\theta_k(U_k,(A \overline{Z}_t+\mathcal G_t)R^{-1}_t)$ with any reversible matrix $R_t$.
\end{proof}

With Lemma \ref{lem:err}, it is easy to derive an analog of Corollary 1.1 in \citep{hardt2014noisy} (see Lemma \ref{lemma1}). We summarize the results as the following lemma.
\begin{lem}
\label{lem:errcor}
Let $k$ and $r$ ($k\leq r$) be the target rank and iteration rank, respectively. Let $U_k\in \mathbb R^{d\times k}$ denote the top $k$ eigenvectors of $A$ and let ${\sigma_1\geq...\geq \sigma_d} $ denote its eigenvalues. Suppose $\overline{Z}_0$ is the orthonormalized space of $\sim \mathcal N(0,\mathbb I_{d\times r})$. Assume the noisy power method iterates as follows,
\begin{equation}
\phantomsection
\overline{Z}_t\leftarrow A \overline{Z}_{t-1}+\mathcal G_t,\nonumber
\end{equation}
where $\overline{Z}_t$ does \emph{not} necessarily have orthonormal columns and $\mathcal G_t$ is some noisy perturbation that satisfies
\begin{equation}
\phantomsection
5\| \mathcal G_t\|_2\leq \epsilon (\sigma_k-\sigma_{k+1}){\rm min}_t\|\overline{Z}_t\|_{\tiny
{\rm m}}\quad{\rm and}\quad 5\|U_k^\intercal \mathcal G_t\|_2\leq (\sigma_k-\sigma_{k+1}){\rm min}_t\|\overline{Z}_t\|_{\tiny
{\rm m}}\frac{\sqrt{r}-\sqrt{k-1}}{\tau \sqrt{d}},\nonumber
\end{equation}
for some fixed $\tau$ and $\epsilon <1/2$. Then with all but $\tau^{-\Omega (r+1-k)}+e^{-\Omega(d)}$ probability, there exists an $T=O(\frac{\sigma_k}{\sigma_k-\sigma_{k+1}}{\rm log}(d\tau/\epsilon))$ so that after $T$ steps $$\|(I-\overline{Z}_T\overline{Z}_T^\intercal)U_k\|_2\leq \epsilon.$$
The result also holds for the sequence
\begin{equation}
\phantomsection
\overline{Z}_t\leftarrow [A \overline{Z}_{t-1}+\mathcal G_t]R^{-1}_t,\nonumber
\end{equation}
with any reversible $R_t$.
\end{lem}

\begin{proof}
By Lemma \ref{lem:err} and the proofing techniques of Corollary 1.1 in \citep{hardt2014noisy} (see Lemma \ref{lemma1}), the result follows.
\end{proof}

In the next lemma, we specify the choice of $R_t$ and analyze the residual error bound $  \|O_{t}^{(i)} R_t - R_{t}^{(i)}O_{t-1}^{(i)} \|_2$ when $t\notin \mathcal I_T^p$. In particular, given a baseline data matrix $A_o$, $R_t$ is the shadow matrix that depicts what the upper triangle matrix ought to be, if we start from the nearest synchronized matrix and perform QR factorization using the matrix $A_o$.
We will set $A_o = A_1$.

\begin{lem}[Choice of $R_t$]
	\label{lemmart}
	Fix any $t$ and let $t_0 = \tau(t) \in \mathcal I_T^p$ be the latest synchronization step before $t$, then $t \ge \tau(t)$.
	\begin{itemize}
		\item If $t = t_0$, we define $R_t = R_t^{(i)}$ for any $i \in [m]$ since all $R_t^{(i)}$'s are equal.
		\item If $t > t_0$, given a baseline data matrix $A_o$, we define $R_t \in \mathbb R^{r \times r}$ recursively as the following. Let $Z_{t_0}=\overline{Z}_{t_0}$.
		For $l = t_0+1, \cdots, t$, we use the following QR factorization to define $R_t$'s:
		\[
		A_o Z_{l} = Z_{l+1} R_{l+1}.
		\]
      With such choice of $R_t$'s, for any $i \in [m]$, we have
      \begin{equation}
	\label{eq:D-bound}
	\|O_{t}^{(i)} R_t - R_{t}^{(i)}O_{t-1}^{(i)}\|_2
	\leq \sigma_1(A_o)\| Z_t^{(i)}O_{t}^{(i)}-  Z_t \|_2+\|A_o - A_i\|_2  + \sigma_1(A_i)\|Z_{t-1}^{(i)}O_{t-1}^{(i)}-Z_{t-1}\|_2.\tag{C.1}
	\end{equation}
	\end{itemize}
\end{lem}

\begin{proof}
	Note that $t \notin \mathcal I_T^p$ and thus $t > t_0$.
	Let's fix some $i \in [m]$ and denote $ \Delta M = M_i - M_o$.
	Based on \textsf{FedPower}, we have for $l = t_0+1, \cdots, t$,
	\[
	  A_i Z_{l}^{(i)} =  Z_{l+1}^{(i)} R_{t+1}^{(i)}.
	\]
	Then,
	\begin{align*}
Z_{l}^{(i)} R_{l}^{(i)}O_{l-1}^{(i)}
	& = A_i  Z_{l-1}^{(i)}O_{l-1}^{(i)}\\
	&= (A_o + \Delta A)( Z_{l-1} + \Delta Z_{l-1} )\\
	&= A_oZ_{l-1} + \Delta A  \cdot Z_{l-1}  + A_i \cdot \Delta Z_{l-1}\\
	&:=A_o Z_{l-1}+ E_{l-1} = Z_l R_l+  E_{l-1}
	\end{align*}
	where $E_{l-1} =  \Delta A  \cdot Z_{l-1}  + A_i \cdot \Delta Z_{l-1}$ and $\Delta Z_{l-1}=Z_{l-1}^{(i)}O_{l-1}^{(i)}  -Z_{l-1} $.
	
	Note that
	\begin{gather*}
	Z_t^{(i)}R_t^{(i)}O_{t-1}^{(i)}= Z_t R_t+  E_{t-1}.
	\end{gather*}
	Then we have
	\begin{align*}
	\|O_{t}^{(i)} R_t - R_{t}^{(i)}O_{t-1}^{(i)} \|_2
	&= 	\|Z_t^{(i)}O_{t}^{(i)} R_t - Z_t^{(i)}R_{t}^{(i)}O_{t-1}^{(i)} \|_2\\
	&\overset{(a)}{=}	\|Z_t^{(i)}O_{t}^{(i)} R_t -  Z_t R_t-  E_{t-1} \|_2\\
	&\leq	\| (Z_t^{(i)}O_{t}^{(i)}-  Z_t ) R_t \|_2 + \| E_{t-1}\|_2 \\
	&\overset{(b)}{\leq}	\| Z_t^{(i)}O_{t}^{(i)}-  Z_t \|_2 \| R_t \|_2 + \|\Delta A\|_2  + \|A_i\|_2\|Z_{t-1}^{(i)}O_{t-1}^{(i)}  -Z_{t-1} \|_2\\
		&\overset{(c)}{\leq}	 \sigma_1(A_o) \| Z_t^{(i)}O_{t}^{(i)}-  Z_t \|_2+ \|A_o - A_i\|_2  + \sigma_1(A_i)\|Z_{t-1}^{(i)}O_{t-1}^{(i)}  -Z_{t-1} \|_2
	\end{align*}
	where (a) uses the equality of $Z_t^{(i)}R_t^{(i)}O_{t-1}^{(i)}$; {(b) uses the definition of $E_{t-1}$}; and (c) uses $\|R_t\|_2 \leq \| A_o\|_2=\sigma_1(A_o)$.
		
\end{proof}

Note that the bound (\ref{eq:D-bound}) in the above lemma depends on the following unknown terms
\[\rho_{t}={\rm max}_i\|Z_t^{(i)}O_t^{(i)}-Z_t\|_2={\rm max}_i\|Z_t^{(i)}D_{t+1}^{(i)}-Z_t\|_2.\]
Hence, in the next two lemmas, we provide the upper bounds for $\rho_{t}$ in two cases, namely, $\mathcal F=\mathcal O_r$ and $\mathcal F=\{\mathbb I_r\}$. Before going on we note that $Z_t$ in $\rho_{t}$ is actually $Z_t^{(1)}$ as $A_o$ is chosen to be $A_1$ in Lemma \ref{lemmart}.
\begin{lem}[Bound for $\rho_{t}$ when $\mathcal F=\mathcal O_r$]
	\label{lem:rho1}
If $D_t^{(i)}$ is solved from
\[D_t^{(i)}=\underset{D\in \mathcal F\cap\mathcal O_r }{{\rm argmin}}\;\|Z_{t-1}^{(i)}D-Z_{t-1}^{(1)}\|_o\]
with $\mathcal F=\mathcal O_r$, where $\|\cdot\|_o$ can be either the Frobenius norm or the spectral norm though in the body text we use only the Frobenius norm,
then
\[
\rho_{t}
\leq {\rm min} \sqrt{2}\left\{{2\kappa^p p \eta (1+\eta)^{p-1} } , \frac{\eta\sigma_1}{\delta_k} + 2\gamma_k^{p/4}  \max_{i \in [m]}\tan \theta_k(Z_{\tau(t)}, U_{k}^{(i)})  \right\}.
\]
where
\begin{itemize}
  \item $\delta_k={\rm min}_{i\in[m]}\delta_k^{(i)}$ with $\delta_k^{(i)} = {\rm min} \{ |\sigma_j(A_i) - \sigma_{k}(A)| :j \geq k+1 \}$;
  \item $\gamma_k={\rm max}\{{\rm max}_{i\in[m]} \frac{\sigma_{k+1}(A_i)}{\sigma_{k}(A_i)},\frac{\sigma_{k+1}(A)}{\sigma_{k}(A)}\}$,
  \item $\kappa = \|A\|_2\|A^\dagger\|_2$ is the condition number of $A$;
  \item $p = t -\tau(u)$, $\tau(t) \in \mathcal I_T^p$ is defined as the nearest synchronization time before $t$.
\end{itemize}
\end{lem}

\begin{proof}
	By Lemma~\ref{lem:orth}, we have
\begin{equation*}
	\|Z_{t-1}^{(i)}D_t^{(i)} - Z_{t-1}^{(1)}\|_2  \leq \sqrt{2}{\rm dist}(Z_{t-1}^{(i)}, Z_{t-1}^{(1)}),
	\end{equation*}
so we only need to bound ${\rm max}_{i \in [m]}{\rm dist}(Z_{t}^{(i)}, Z_{t}^{(1)})$.
	We will bound each ${\rm dist}(Z_{t}^{(i)}, Z_{t}^{(1)})$ uniformly in two ways.
	Then the minimum of the two upper bounds holds for their maximum that is exactly $\rho_{t}$.
	Fix any $i \in [m]$ and $t \in [T]$.
		Let $\tau(t)$ be the latest synchronization step before $t$ and $p=t-\tau(t)$ be the number of nearest local updates.
	\begin{itemize}
		\item For small $p$, by Lemma~\ref{lem:PT-PJ}, it follows that
		\begin{align*}
		{\rm dist}(Z_t^{(i)}, Z_t^{(1)})
		&= {\rm dist}(A_i^pZ_{\tau(t)}, A_1^pZ_{\tau(t)})\\
		&\leq {\rm dist}(A_i^pZ_{\tau(t)}, A^pZ_{\tau(t)})
		+ {\rm dist}(A^pZ_{\tau(t)}, A_1^pZ_{\tau(t)})\\
		&\leq  {\rm min}\{ \|(A_i^pZ_{\tau(t)})^\dagger\|_2, \|(A^pZ_{\tau(t)})^\dagger\|_2 \} \|(A_i^p-A^p)Z_{\tau(t)}\|_2\\
		& \qquad + {\rm min}\{ \|(A^pZ_{\tau(t)})^\dagger\|_2, \|(A_1^pZ_{\tau(t)})^\dagger\|_2 \} \|(A^p-A_1^p)Z_{\tau(t)}\|_2\\
		&\leq{2\kappa^p ((1+\eta)^p-1)\leq 2\kappa^p p \eta (1+\eta)^{p-1}},
		\end{align*}
		where $\kappa = \|A\|_2\|A^\dagger\|_2$ is the condition number of $A$.
		\item  For large $p$, let the top-$k$ eigenspace of $A_1$ and $A_i$ be respectively $U_k^{(1)}$ and $U_{k}^{(i)}$ (both of which are orthonormal).
		The $k$-largest eigenvalue of $A$ is denoted by $\sigma_{k}(A_1)$ and similarly for $\sigma_{k}(A_i)$.
		Then by Lemma~\ref{lem:DK}, we have 	\[
		{\rm dist}(U_k, U_{k}^{(i)})
		\leq \frac{\|A_i - A\|_2 }{\delta_k^{(i)}}
		\leq    \frac{\eta\sigma_1}{\delta_k^{(i)}}.
		\]
		where $\sigma_1 = \sigma_1(A)$ and $\delta_k^{(i)} = {\rm min} \{ |\sigma_j(A_i) - \sigma_{k}(A)| :j \geq k+1 \}$.
		
		Note that local updates are equivalent to noiseless power method.
		Then, using Lemma~\ref{lemhardt} and setting $\epsilon = 0$ and $\mathcal G = 0$ therein, we have
		\[
		{\rm tan} \theta_k(Z_t^{(i)}, U_{k}^{(i)})  \leq
		\left( \frac{\sigma_{k+1}(A_i)}{\sigma_{k}(A_i)} \right)^{1/4}
		{\rm tan} \theta_k(Z_{t-1}^{(i)}, U_{k}^{(i)}).
		\]
		
		Hence,
{\small
		\begin{align*}
		{\rm dist}(Z_t^{(i)}, Z_t^{(1)})
		&\leq {\rm dist}(Z_t^{(i)}, U_k^{(i)}) +{\rm dist}(U_k^{(i)}, U_k^{(1)}) +{\rm dist}(U_k^{(1)}, Z_t^{(1)}) \\
		&\leq   \left( \frac{\sigma_{k+1}(A_i)}{\sigma_{k}(A_i)}\right)^{p/4}	{\rm tan} \theta_k(Z_{\tau(t)}, U_{k}^{(i)})   + \frac{\eta\sigma_1}{\delta_k^{(i)}} + \left( \frac{\sigma_{k+1}(A)}{\sigma_{k}(A)}\right)^{p/4}	{\rm tan} \theta_k(Z_{\tau(t)}, U_{k}^{(1)}) \\
		&\leq  \frac{\eta\sigma_1}{{\rm min}_{i \in [m]}\delta_k^{(i)}} + 2 \gamma_k^{p/4}  {\rm max}_{i \in [m]}{\rm tan} \theta_k(Z_{\tau(t)}, U_{k}^{(i)}).
		\end{align*}}
	\end{itemize}

   Combining the two cases, we have
   \[
   \rho_t
   \leq \sqrt{2}\min \left\{{2\kappa^p p \eta (1+\eta)^{p-1} }, \frac{\eta\sigma_1}{\delta_k} + 2 \gamma_k^{p/4}  \max_{i \in [m]}{\rm tan} \theta_k(Z_{\tau(t)}, U_{k}^{(i)})  \right\}.
   \]

\end{proof}

\begin{lem}[Bound for $\rho_{t}$ when $\mathcal F=\{\mathbb I_r\}$]
	\label{lem:rho2}
	If $D_t^{(i)}$ is solved from
\[D_t^{(i)}=\underset{D\in \mathcal F\cap\mathcal O_r }{{\rm argmin}}\;\|Z_{t-1}^{(i)}D-Z_{t-1}^{(1)}\|_o,\]
with $\mathcal F = \{\mathbb I_r\}$, then
	\[
	\rho_t \leq 4\sqrt{2k}p \kappa^p \eta (1+\eta)^{p-1},
	\]
where $\kappa = \|A\|_2\|A^\dagger\|_2$ is the condition number of $A$, $p = t -\tau(u)$, $\tau(t) \in \mathcal I_T^p$ is defined as the nearest synchronization time before $t$.
\end{lem}

\begin{proof}
	In this case, we are going to bound $\rho_{t} = {\rm max}_{i \in [m]}\|Z^{(i)}-Z_t^{(1)}\|_2$.
	Fix any $i \in [m]$ and $t \in [T]$.
	We will bound $\|Z^{(i)}-Z_t^{(1)}\|_2$ uniformly so that the bound holds for their maximum.
	
	Fix any $i \in [m]$ and $t \in [T]$.
	Let $\tau(t)$ be the latest synchronization step before $t$ and $p=t-\tau(t)$ be the number of nearest local updates.
	Note that $Z_{t}^{(i)}$ and $Z_{t}^{(1)}$ are the $Q$-factor of the QR factorization of $A_i^pZ_{\tau(t)}$ and $A_1^pZ_{\tau(t)}$.
	Let $\tilde{Z}_t$ be the  $Q$-factor of the QR factorization of $A^pZ_{\tau(t)}$.
	Then Lemma~\ref{lem:diff_r} yields
	\[
	\|Z_{t}^{(i)}- \tilde{Z}_{t}\|_2
	\leq  \sqrt{2k} \frac{\|(A^pZ_{\tau(t)})^\dagger\|_2\|(A_i^p-A^p)Z_{\tau(t)} \|_2 }{1- \|(A^pZ_{\tau(t)})^\dagger\|_2\|(A_i^p-A^p)Z_{\tau(t)} \|_2}
	:= \sqrt{2k}  \frac{\omega}{1-\omega}
	\]
	where $\omega = \|(A^pZ_{\tau(t)})^\dagger\|_2\|(A_i^p-A^p)Z_{\tau(t)} \|_2$ for short.
	If $\omega \leq 1/2$, then we have $\|Z_{t}^{(i)}- \tilde{Z}_{t}\|_2  \leq 2\sqrt{2k} \omega$.
	Otherwise, we have $\omega \geq 1/2$ and $\|Z_{t}^{(i)}- \tilde{Z}_{t}\|_2  \leq 2 \leq \sqrt{2k} \leq 2\sqrt{2k} \omega$.
	Then we have for all $i \in [m]$,
	\[
		\|Z_{t}^{(i)}- \tilde{Z}_{t}\|_2  \leq 2\sqrt{2k} \|(A^pZ_{\tau(t)})^\dagger\|_2\|(A_i^p-A^p)Z_{\tau(t)} \|_2.
	\]
	
	Hence,
	\begin{align*}
		\rho_t &=  \|Z_{t}^{(i)}- Z_{t}^{(1)}\|_2 \\
		&\leq \|Z_{t}^{(i)}- \tilde{Z}_{t}\|_2  + \|\tilde{Z}_{t}- {Z}_{t}^{(1)}\|_2 \\
		&\leq 2\sqrt{2k}  \left[ \|(A^pZ_{\tau(t)})^\dagger\|_2\|(A_i^p-A^p)Z_{\tau(t)} \|_2 +
		\|(A^pZ_{\tau(t)})^\dagger\|_2\|(A_1^p-A^p)Z_{\tau(t)} \|_2 \right]\\
		&\leq 4\sqrt{2k} \kappa^p \left[ (1+\eta)^p -1 \right]\\
			&\leq 4\sqrt{2k}p \kappa^p \eta (1+\eta)^{p-1},
	\end{align*}
			where $\kappa = \|A\|_2\|A^\dagger\|_2$ is the condition number of $A$.
\end{proof}

The next lemma provide a lower bound for $\|\overline{Z}_t\|_{\rm \tiny m}$, which is needed when using Lemma \ref{lem:err} to carry out the convergence analysis of \textsf{FedPower}.
\begin{lem}[Bound for $\|\overline{Z}_t\|_{\rm \tiny m}$]
	\label{lem:z}
Recall that
\[\rho_{t}={\rm max}_i\|Z_t^{(i)}O_t^{(i)}-Z_t^{(1)}\|_2.\]
Then the following holds\\
(a) If $\overline{Z}_t:= \frac{1}{m}\sum_{i=1}^m Z_t^{(i)}O_t^{(i)}$, then
\[\|\overline{Z}_t\|_{\rm \tiny m}\geq 1-(1-1/m){\rm max}_t \rho_t:=\mu_{t1};\]
(b) If $\overline{Z}_t:= \frac{1}{K}\sum_{i\in \mathcal S_{\tau(t)}} Z_t^{(i)}O_t^{(i)}$, then
\[\|\overline{Z}_t\|_{\rm \tiny m}\geq1-{\rm max}_t \rho_t:= \mu_{t2}.\]
\end{lem}
\begin{proof}
It suffices to show $\|\overline{Z}_t^\dagger\|_2\leq 1/\mu_t$ ($\mu_t=\mu_{t1},\mu_{t2},\mu_{t3}$) by noting $\|\overline{Z}_t\|_{\rm \tiny m}\|\overline{Z}_t^\dagger\|_2=1$. Next we show (a), (b) and (c), respectively.

(a) For $\overline{Z}_t= \frac{1}{m}\sum_{i=1}^m Z_t^{(i)}O_t^{(i)}$, we have
	\[\|\overline{Z}_{t}^\dagger \|_2= {\rm max}\{ \|w\|_2 :  \|\frac{1}{m}\sum_{i=1}^mZ_t^{(i)}O_t^{(i)}w \|_2 \leq 1  \}. \]
When $t\in \mathcal I_T^p$, $O_t^{(i)}=\mathbb I$ and $Z_t^{(i)}$'s are equal, hence $\|\overline{Z}_t^\dagger\|_2=1$ and the result holds naturally. When $t\notin \mathcal I_T^p$, we have
\begin{align*}
	\|\frac{1}{m}\sum_{i=1}^mZ_t^{(i)}O_t^{(i)}w \|_2 &=\|\frac{1}{m}\sum_{i=1}^mZ_t^{(i)}D_{t+1}^{(i)}w \|_2\\
	&\geq \|\frac{1}{m}\sum_{i=1}^mZ_t^{(1)}w \|_2-\|\frac{1}{m}\sum_{i=1}^m(Z_t^{(i)}D_{t+1}^{(i)}-Z_t^{(1)})w\|_2\\
&{\geq}\|w\|_2 (1-\frac{1}{m}\sum_{i=1}^m\|Z_t^{(i)}D_{t+1}^{(i)}-Z_t^{(1)}\|_2)
\geq\|w\|_2\mu_{t1}.	
	\end{align*}
Hence, $\|\overline{Z}_{t}^\dagger \|_2\leq 1/\mu_{t1}$.

(b) For $\overline{Z}_t=\frac{1}{K}\sum_{i\in \mathcal S_{\tau(t)}} Z_t^{(i)}O_t^{(i)}$, we have
\begin{align*}
	\|\frac{1}{K}\sum_{i\in \mathcal S_{\tau(t)}} Z_t^{(i)}O_t^{(i)}w \|_2 &=\|\frac{1}{K}\sum_{i\in \mathcal S_{\tau(t)}}Z_t^{(i)}D_{t+1}^{(i)}w \|_2\\
	&\geq \|\frac{1}{K}\sum_{i\in \mathcal S_{\tau(t)}}Z_t^{(1)}w \|_2-\|\frac{1}{K}\sum_{i\in \mathcal S_{\tau(t)}}(Z_t^{(i)}D_{t+1}^{(i)}-Z_t^{(1)})w\|_2\\
&\geq\|w\|_2 (1-\frac{1}{K}\sum_{i\in \mathcal S_{\tau(t)}}\|Z_t^{(i)}D_{t+1}^{(i)}-Z_t^{(1)}\|_2)\\
&\geq\|w\|_2\mu_{t2},	
	\end{align*}
then the result follows.
 	
\end{proof}

\subsection*{D. Auxiliary lemmas}
\label{appen:d}

\begin{lem}[Lemma 2.2 of \citep{hardt2014noisy}]
\label{lemhardt}
Let $U_k\in \mathbb R^{d\times k}$ denote the top $k$ eigenvectors of $A$ and let $\sigma_1\geq...\geq \sigma_d $ denote its eigenvalues. Let $Z\in \mathbb R^{ d\times r}$ with $Z^\intercal Z=\mathbb I_r$ for some $r\geq k$. Let $\mathcal G$ satisfy
\begin{equation}
\phantomsection
4\|U_k^\intercal \mathcal G\|_2\leq (\sigma_k-\sigma_{k+1}){\rm cos}\,\theta_k(U_k, Z)\quad{\rm and}\quad 4\| \mathcal G\|_2\leq (\sigma_k-\sigma_{k+1})\epsilon,\nonumber
\end{equation}
for some $\epsilon\leq 1$. Then
$${\rm tan}\,\theta_k(U_k,A Z+\mathcal G )\leq {\rm max} \left(\epsilon,{\rm max}\left(\epsilon,\left(\frac{\sigma_{k+1}}{\sigma_k}\right)^{1/4}\right){\rm tan}\,\theta_k(U_k,Z)\right).$$
\end{lem}

\begin{lem}[Corollary 1.1 of \citep{hardt2014noisy}]
\label{lemma1}
Let $k$ and $r$ ($k\leq r$) be the target rank and iteration rank, respectively. Let $U_k\in \mathbb R^{d\times k}$ denote the top $k$ eigenvectors of $A$ and let $\sigma_1\geq...\geq \sigma_d $ denote its eigenvalues. Suppose $Z_0$ be the orthonormalized space of $\mathcal N(0,\mathbb I_{d\times r})$. Assume the noisy power method iterates as follows,
\begin{equation}
\phantomsection
Y_t\leftarrow A Z_{t-1}+\mathcal G_t \quad{\rm and}\quad Z_{t}\leftarrow {\textsf{\rm orth}}(Y_t),\nonumber
\end{equation}
where $Z_t\in \mathbb R^{ d\times r}$ with $Z_t^\intercal Z_t=\mathbb I_r$ and $\mathcal G_t$ is some noisy perturbation that satisfies
\begin{equation}
\phantomsection
5\| \mathcal G_t\|_2\leq \epsilon (\sigma_k-\sigma_{k+1})\quad{\rm and}\quad 5\|U_k^\intercal \mathcal G_t\|_2\leq (\sigma_k-\sigma_{k+1})\frac{\sqrt{r}-\sqrt{k-1}}{\tau \sqrt{d}},\nonumber
\end{equation}
for some fixed $\tau$ and $\epsilon <1/2$. Then with all but $\tau^{-\Omega (r+1-k)}+e^{-\Omega(d)}$ probability, there exists an $T=O(\frac{\sigma_k}{\sigma_k-\sigma_{k+1}}{\rm log}(d\tau/\epsilon))$ so that after $T$ steps $$\|(I-Z_TZ_T^\intercal)U_k\|_2\leq \epsilon.$$
\end{lem}

\begin{lem}
	\label{lem:diff_r}
	Let $A \in \mathbb R^{d \times k}$ with $d \geq  k$ be any matrix with full rank.
	Denote by its QR factorization as $A = Q R$ where $Q$ is an orthogonal matrix.
	Let $E$ be some perturbation matrix and $A + E = \tilde{Q} \tilde{R}$ the resulting QR factorization of $A + E$.
	When $\|E\|_2 \|A^{\dagger}\|_2 < 1$, $A + E$ is of full rank.
	What's more, it follows that
	\[    \| \tilde{Q} - Q\|_2
	\leq
	\sqrt{2k}\frac{\|A^\dagger\|_2\|E\|_2}{1-\|A^\dagger\|_2\|E\|_2}. \]
\end{lem}
\begin{proof}
	Actually, we have
		\[    \| \tilde{Q} - Q\|_F \overset{(a)}{\leq}
	\frac{\sqrt{2}\|E\|_F}{\|E\|_2} {\rm ln} \frac{1}{1-\|A^\dagger\|_2\|E\|_2}
	\overset{(b)}{\leq}
	\sqrt{2}\frac{\|A^\dagger\|_2\|E\|_F}{1-\|A^\dagger\|_2\|E\|_2}
	\overset{(c)}{\leq}
	\sqrt{2k}\frac{\|A^\dagger\|_2\|E\|_2}{1-\|A^\dagger\|_2\|E\|_2} \]
	where (a) comes from Theorem 5.1 in~\citep{sun1995perturbation},
	(b) uses ${\rm ln}(1+x) \leq x$ for all $x > -1$,
	and (c) uses $\|E\|_F \leq \sqrt{k}\|E\|_2$.
\end{proof}

\begin{defn}[Bernstein tail condition]
\label{bernstein}
A random variable $Y$ is said to be $(v,R)$-Bernstein if $\mathbb E[|Y|^k]\leq \frac{v}{2}k!R^{k-2}$ for all integers $k\geq 2$.
\end{defn}

\begin{lem}[Theorem 3 in \citep{lei2022bias}]
\label{lei-bound1}
 Let $X_1,...,X_m$ be a sequence of independent $d\times r$ random matrices with zero-mean independent entries being $(v_1, R_1)$-Bernstein (see Definition \ref{bernstein}). Let $H_1,...H_m$ be any sequence of $r\times r$ non-random matrices. Then for all $t>0$,
$$\mathbb P(\|\sum_{i=1}^m X_iH_i\|_2\geq t)\leq 4d\exp\left(-\frac{t^2/2}{v_1d\|\sum_i H_i^\intercal H_i\|_2+R_1 \max_i\|H_i\|_{2\rightarrow\infty} t}\right).$$
\end{lem}

\begin{lem}[Absolute moments of Gaussian \citep{winkelbauer2012moments}]
\label{momentofgaussian}
Let $Y$ be zero-mean Gaussian random variable with variance $\nu$. Then for any non-negative integer $k$,
\[ \mathbb E (|Y|^k)=\nu^k (k-1)!! \begin{cases}
\sqrt{\frac{\pi}{2}}\; &{\rm if}\; k\; {\rm is}\; {\rm odd},\\
\;\;1 \; &{\rm if}\; k\; {\rm is}\; {\rm even}.
\end{cases}
\]
\end{lem}

\begin{lem}[Matrix Bernstein inequality, Chapter 6 in \citep{tropp2015introduction}]
\label{lemmabernstein}
Consider a finite sequence $\{S_k\}$ of independent, random matrices with common dimension $d_1\times d_2$. Assume that
$$\mathbb E S_k=0\quad \mbox{ and} \quad \|S_k\|_2\leq L\quad \mbox{ for each index } k. $$
Introduce the random matrix $$Z=\sum_k S_k.$$
Let $\nu(Z)$ be the matrix variance statistics of the sum:
$$\nu(Z)={\rm max}\{\|\mathbb E(ZZ^\intercal)\|_2, \|\mathbb E(Z^\intercal Z)\|_2\}. $$
Then $$\mathbb E\|Z\|_2\leq \sqrt{2\nu(Z){\rm log}(d_1+d_2) }+\frac{1}{3}L{\rm log}(d_1+d_2).$$
Moreover, for all $t\geq 0$,
$$\mathbb P(\|Z\|_2\geq t)\leq (d_1+d_2){\rm exp}\left(\frac{-t^2/2}{\nu(Z)+Lt/3}\right).$$
\end{lem}

\begin{lem}[Davis-Kahan $\sin(\theta)$ theorem]
	\label{lem:DK}
	Let the top-$k$ eigenspace of $A$ and $\tilde{A}$ be respectively $U_k$ and $\tilde{U}_{k}$ (both of which are orthonormal).
	The $k$-largest eigenvalue of $A$ is denoted by $\sigma_{k}(A)$ and similarly for $\sigma_{k}(\tilde{A})$.
	Define $\delta_k = {\rm min} \{ |\sigma_k(A) - \sigma_{j}(\tilde{A})| :j \geq k +1 \}$, then
	\[
	{\rm dist}(U_k, \tilde{U}_{k}) = {\rm sin} \theta_k(U_k, \tilde{U}_{k}) \leq \frac{\| A - \tilde{A}\|_2}{\delta_k}.
	\]
\end{lem}

\begin{lem}[Perturbation theorem of projection distance]
	\label{lem:PT-PJ}
	\label{lem:PT}
	Let $\mathrm{rank}(X) = \mathrm{rank}(Y)$, then
	\begin{equation*}
	{\rm dist}(X, Y) \leq {\rm min}\{ \|X^\dagger\|_2, \|Y^\dagger\|_2 \} \|X-Y\|_2.
	\end{equation*}
\end{lem}
\begin{proof}
See Theorem 2.3 of~\citep{ji1987perturbation}.
\end{proof}

\begin{lem}[Uniform sampling]
	\label{lem:uniform}
Let $\eta, \zeta \in (0, 1)$. Assume the rows of $B_i$ are sampled from the rows of $B$ uniformly at random. Assume each node has sufficiently many samples, that is, for all $i \in m$,
\begin{equation*}
s_i
	\: \geq \: \tfrac{3 \mu k}{\epsilon^2}\log \big( \tfrac{\rho m}{{\zeta}} \big),
\end{equation*}
where $s_i$ is the number of rows of $B_i$, $k$ is the rank of $B$ and $\mu$ is the row coherence of $B$. Then with probability greater than $1 - \zeta$, $\eta$ in Definition \ref{assu1} is smaller than $\epsilon$.
\end{lem}

\subsection*{E. Definitions on subspace distance}
\label{appen:e}
In this subsection, we introduce additional definitions and lemmas on metrics between two subspaces.
Let $\mathcal O_{d \times k}$ be the set of all $d \times k$ orthonormal matrices and $\mathcal O_k$ short for $\mathcal O_{k \times k}$ denote the set of $k \times k$ orthogonal matrices.

\paragraph{Principle Angles.}
Given two matrices $U, \tilde{U} \in \mathcal O_{d \times k}$ which are both full rank with $1 \leq k \leq d$, we define the $i$-th ($1 \leq i \leq k$) between $U$ and $\tilde{U}$ in a recursive manner:
\begin{equation}
\label{eq:theta}
\theta_i(U, \tilde{U})  = {\rm min} \bigg\{  \arccos\left(  \frac{x^\intercal y}{\|x\|_2\|y\|_2} \right): x \in \mathcal R(U), y \in \mathcal R(\tilde{U}), x \perp x_j, y \perp y_j, \forall  j < i  \bigg\},\nonumber
\end{equation}
where $\mathcal R(U)$ denotes by the space spanned by all columns of $U$.
In this definition, we require that $0 \leq \theta_1 \leq \cdots \leq \theta_k \leq \frac{\pi}{2}$ and that $\{x_1, \cdots, x_k\}$ and $\{y_1, \cdots, y_k\}$ are the associated principal vectors.
Principle angles can be used to quantify the similarity between two given subspaces.

We have following facts about the $k$-th principle angle between $U$ and $\tilde{U}$:
\begin{fact}
\label{fact}
Let $U^{\perp}$ denote by the complement subspace of $U$ (so that $[U, U^{\perp}] \in \mathbb R^{d \times d}$ forms an orthonormal basis of $\mathbb R^d$) and so does $\tilde{U}^{\perp}$,
	\begin{enumerate}
		\item $ {\rm sin}\, \theta_k(U, \tilde{U}) =  \| U^\intercal \tilde{U}^{\perp}\|_2 = \| \tilde{U}^\intercal U^{\perp}\|_2$;
		\item $ {\rm tan}\, \theta_k(U, \tilde{U}) = \| \left[ (U^{\perp})^{\intercal}\tilde{U}\right] (U^{\intercal} \tilde{U})^{\dagger}\|_2 $ where $\dagger$ denotes by the Moore-Penrose inverse.
		\item For any reversible matrix $R \in \mathbb R^{k \times k}$, $ {\rm tan}\, \theta_k(U, \tilde{U}) = {\rm tan} \,\theta_k(U, \tilde{U} R)$.
	\end{enumerate}
\end{fact}
\begin{fact}
\label{fact2}
Let $U^{\perp}$ denote by the complement subspace of $U$ (so that $[U, U^{\perp}] \in \mathbb R^{d \times d}$ forms an orthonormal basis of $\mathbb R^d$) and so does $\tilde{U}^{\perp}$,
	\begin{enumerate}
		\item $ {\rm sin}\, \theta_k(U, \tilde{U}) =  \| U^\intercal \tilde{U}^{\perp}\|_2 = \| \tilde{U}^\intercal U^{\perp}\|_2$;
		\item $ {\rm tan}\, \theta_k(U, \tilde{U}) = \| \left[ (U^{\perp})^{\intercal}\tilde{U}\right] (U^{\intercal} \tilde{U})^{\dagger}\|_2 $ where $\dagger$ denotes by the Moore-Penrose inverse.
		\item For any reversible matrix $R \in \mathbb R^{k \times k}$, $ {\rm tan}\, \theta_k(U, \tilde{U}) = {\rm tan} \,\theta_k(U, \tilde{U} R)$.
	\end{enumerate}
\end{fact}
\paragraph{Projection Distance.}
  Define the projection distance\footnote{Unlike the spectral norm or the Frobenius norm, the projection norm will not fall short of accounting for global orthonormal transformation. Check~\citep{ye2016schubert} to find more information about the distance between two spaces.} between two subspaces by
\begin{equation}
\label{eq:dist}
{\rm dist}(U, \tilde{U}) = \| U U^\intercal - \tilde{U} \tilde{U}^\intercal\|_2.\nonumber
\end{equation}
This metric has several equivalent expressions:
\[ {\rm dist}(U, \tilde{U}) = \| U^\intercal \tilde{U}^{\perp}\|_2 = \| \tilde{U}^\intercal U^{\perp}\|_2 ={\rm sin} \theta_k(U, \tilde{U}).   \]
More generally, for any two matrix $A, B \in \mathbb R^{d \times k}$, we define the projection distance between them as
\[
{\rm dist}(A, B) = \| U_A U_A^\intercal - U_B U_B^\intercal\|_2,
\]
where $U_A, U_B$ are the orthogonal basis of $\mathcal R(A)$ and $\mathcal R(B)$ respectively.

\paragraph{Orthogonal Procrustes.}
Let $U, \tilde{U} \in \mathbb R^{d \times k}$ be two orthonormal matrices.
$\mathbb R(U)$ is close to $\mathbb R(\tilde{U})$ does not necessarily imply $U$ is close to $\tilde{U}$, since any orthonormal invariant of $U$ forms a base of $\mathcal R(U)$.
However, the converse is true.
If we try to map $\tilde{U}$ to $U$ using an orthogonal transformation, we arrive at the following optimization
\begin{equation}
\label{eq:best_O}
O^* = {\rm argmin}_{O \in \mathcal O_r} \|U -\tilde{U} O\|_F, \nonumber
\end{equation}
where $\mathcal O_k$ denotes the set of $k \times k$ orthogonal matrices.
The following lemma shows there is an interesting relationship between the subspace distance and their corresponding basis matrices.
It implies that as a metric on linear space, ${\rm dist}(U, \tilde{U})$ is equivalent to $\|U -\tilde{U} O^*\|_2$ (or ${\rm min}_{O \in \mathcal O_k} \|U -\tilde{U} O\|_2$) up to some universal constant.
The optimization problem involved in is named as the orthogonal procrustes problem and has been well studied~\citep{schonemann1966generalized,cape2020orthogonal}.

\begin{lem}
	\label{lem:orth}
	Let $U, \tilde{U} \in \mathcal O_{d \times k}$ and $O^*$ is the solution of the following optimization,
$$O^*= {\rm argmin}_{O \in \mathcal O_k} \|U -\tilde{U} O\|_F,$$
	Then we have
	\begin{enumerate}
		\item $O^*$ has a closed form given by $O^* = W_1W_2$ where $\tilde{U}^\intercal U = W_1 \Sigma W_2$ is the singular value decomposition of $\tilde{U}^\intercal U$.
		\item Define $d(U, \tilde{U}) :=  \|U -\tilde{U} O^*\|_2$ where $\|\cdot\|_2$ is the spectral norm.
		Then we have
		\[
		d(U, \tilde{U}) = \sqrt{2-2\sqrt{1- {\rm dist}(U, \tilde{U})^2 }}  = 2{\rm sin}\frac{ \theta_k(U, \tilde{U})}{2}.
		\]
		\item $d(U_1, U_2) = d(U_2, U_1)$ for any $U_1, U_2 \in \mathcal O_{d \times k}$.
		\item ${\rm dist}(U, \tilde{U}) \leq  d(U, \tilde{U}) \leq \sqrt{2} {\rm dist}(U, \tilde{U}) $.
		\item Define
		\[
		\ell(U, \tilde{U}) :=  {\rm min}_{O \in \mathcal O_k}\|U -\tilde{U} O\|_2.
		\]
		Then $\ell(U, \tilde{U})$ is a metric satisfying
		\begin{itemize}
			\item $\ell(U, \tilde{U}) \geq 0$ for all $U, \tilde{U} \in \mathcal O_{d \times k}$.
			$\ell(U, \tilde{U}) = 0$ if and only if $\mathcal R(U) = \mathcal R(\tilde{U})$.
			\item $\ell(U, \tilde{U}) = \ell(\tilde{U}, U) $ for all $U, \tilde{U} \in \mathcal O_{d \times k}$.
			\item $\ell(U_1, U_2) \leq \ell(U_1, U_3) + \ell(U_3, U_2) $ for any $U_1, U_2$ and $U_3 \in \mathcal O_{d \times k}$.
		\end{itemize}
		\item ${\rm dist}(U, \tilde{U}) \leq \ell(U, \tilde{U}) \leq  d(U, \tilde{U}) \leq \sqrt{2} {\rm dist}(U, \tilde{U})$.
	\end{enumerate}
\end{lem}

\begin{proof}
	The first item comes from~\citep{schonemann1966generalized}.
	The second item comes from~\citep{cape2020orthogonal}.
	The third and fourth items follow from the second one.
	The fifth item follows directly from the definition.
	For the rightest two $\leq$ of the last item, we use $\ell(U, \tilde{U}) \leq d(U, \tilde{U})$ and the fourth item.
	For the leftest $\leq$, see Lemma 2.6 in~\citep{chen2021spectral} and Proposition 2.2 of~\citep{vu2013minimax}.
\end{proof}

\vskip 0.2in
\bibliographystyle{plainnat}
\bibliography{FedSVD}
\end{document}